\documentclass[letterpaper]{article} 
\usepackage[preprint]{aaai2027}  
\usepackage[hyphens]{url}  
\usepackage{graphicx} 
\usepackage{natbib}  
\usepackage{caption} 
\usepackage{algorithm}
\usepackage{xspace}
\usepackage{amsmath,amssymb,amsfonts,amsthm}
\usepackage{multirow}
\usepackage[table]{xcolor}
\usepackage{algpseudocode}

\usepackage{newfloat}
\usepackage{listings}
\DeclareCaptionStyle{ruled}{labelfont=normalfont,labelsep=colon,strut=off} 
\floatstyle{ruled}
\newfloat{listing}{tb}{lst}{}
\floatname{listing}{Listing}

\usepackage{booktabs}

\newcommand{\name}{\textsc{DualShield}\xspace}

\newcommand{\bestcell}[1]{\cellcolor{black!8}\textbf{#1}}
\newtheorem{theorem}{Theorem}

\title{\textit{Double Down on Defense}: Strengthening Deep Perceptual Hashes\\against Evasion Attacks without Retraining}
\author{
    Bangjie Sun\textsuperscript{\rm 1},
    Nayoung Kim\textsuperscript{\rm 2},
    Mun Choon Chan\textsuperscript{\rm 1},
    Jun Han\textsuperscript{\rm 2}
}
\affiliations{
\textsuperscript{\rm 1}National University of Singapore (NUS)\\
\textsuperscript{\rm 2}Korea Advanced Institute of Science and Technology (KAIST)\\
bangjie@nus.edu.sg, skdud@kaist.ac.kr, dcscmc@nus.edu.sg, junhan@cyphy.kaist.ac.kr
}

\begin{document}

\maketitle

\begin{abstract}
Near-duplicate image matching is essential for trust and safety, provenance verification, copyright enforcement, and large-scale visual search. To perform this task efficiently, modern online platforms increasingly use deep perceptual hashes, which map visually similar images to similar representations even after common image transformations. However, these systems remain vulnerable to adversarial manipulations designed to make a near-duplicate appear different and evade matching. We present \name, a plug-in defense that strengthens the robustness of existing deep hashes without retraining or modifying their underlying hash models. \name combines two complementary mechanisms: \emph{matching-time randomized smoothing}, which strengthens the match/non-match decision by repeatedly comparing slightly perturbed versions of the reference and query images, and \emph{publication-time hardening}, which adds an optimized imperceptible perturbation to each reference image before publication. Together, these mechanisms provide both certified and empirical robustness. For \emph{certified robustness}, \name establishes a certified $\ell_2$ radius of approximately 0.3 for near-duplicate matches, formally guaranteeing that any adversarial perturbation to the query image within this radius cannot evade near-duplicate matching. For \emph{empirical robustness}, we evaluate \name against adaptive white-box and black-box adversarial attacks, as well as image-transformation attacks, and measure its ability to maintain low attack success rates. Across eight deep perceptual hashes and three datasets, \name substantially improves robustness against evasion attacks while preserving low hash collision rates. These results demonstrate that existing deep perceptual hashes can be substantially strengthened without costly model retraining, by optimizing the hash matching procedure and imperceptibly hardening reference images before publication.
\end{abstract}


\section{Introduction}

Near-duplicate image matching is widely used in online content moderation, media provenance, digital forensics, and large-scale image search~\citep{zhou2025survey,singhi2025provenance,yandex_images,tineye_faq_howworks,du2020perceptual,liu2023robust,samanta2024smarthash,marcel2024perceptual}. For example, platforms use it to detect re-uploads of known harmful content, while provenance and forensic systems use it to identify copied or modified media. These applications increasingly rely on deep perceptual hashes, which encode image content into compact representations (i.e., hashes) that can be efficiently compared at scale. Unlike cryptographic hashes such as SHA, perceptual hashes are designed so that images with similar visual content produce hashes that are close to one another. A typical near-duplicate matching system computes a deep perceptual hash for a query image and compares it against registered reference hashes using a similarity threshold. However, an adversary can deliberately make imperceptible changes to a known image so that its hash crosses the matching boundary and no longer matches the stored reference (i.e., \textit{evasion}). For example, a user may alter previously removed harmful content to re-upload it without detection, or modify a copyrighted image to avoid being identified by copyright-monitoring systems. Therefore, defending against such evasion attacks remains an important research problem, requiring both a robust hash representation and a reliable matching rule.

\begin{figure*}[t]
  \centering
  \includegraphics[width=.95\textwidth]{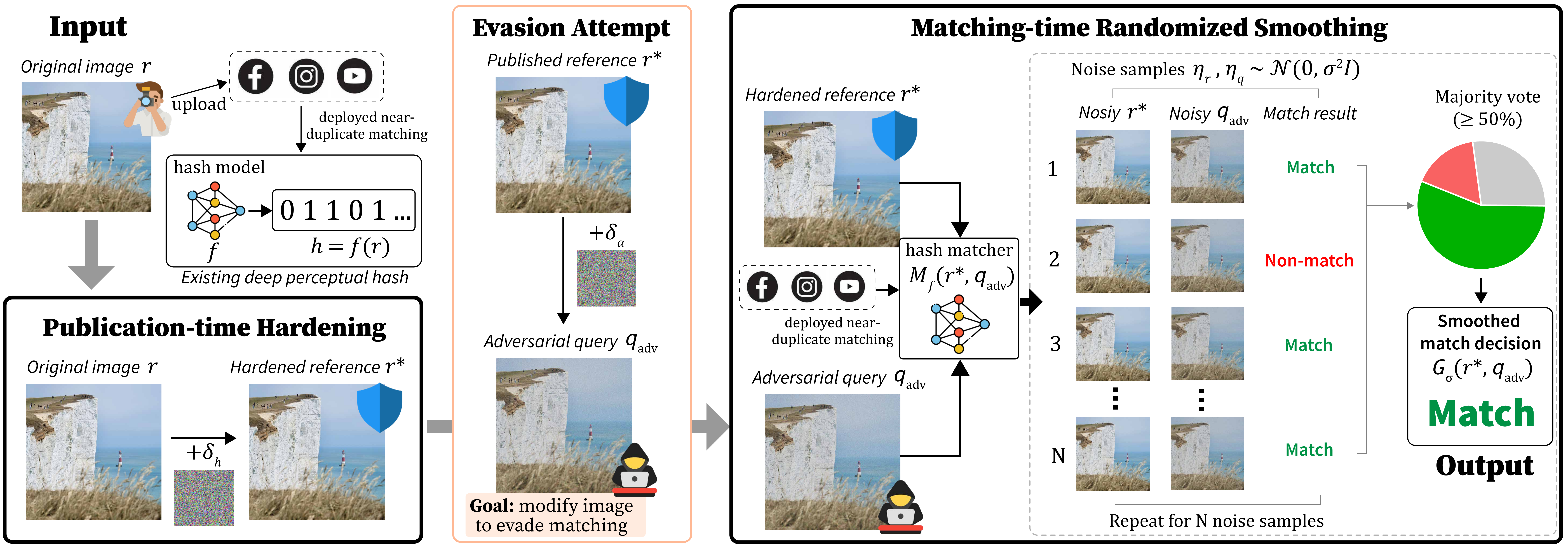}
  \caption{Consider an online platform where a published image is later used as a reference for detecting modified re-uploads. \name provides two complementary ``shields''. The first, \emph{publication-time hardening}, adds an imperceptible perturbation before publication to create a more robust reference, which is the only version accessible to an adversary. The second, \emph{matching-time randomized smoothing}, is applied when another user, who may be an adversary, uploads a query image. It repeatedly compares perturbed versions of the reference-query pair and aggregates their outcomes to produce match/non-match decisions.}
  \label{fig:overview}
\end{figure*}

Deep perceptual hashes can be evaluated under two forms of robustness: \emph{empirical robustness} and \emph{certified robustness}. Empirical robustness measures whether near-duplicate images remain correctly matched under a predefined set of image modifications. These include common photometric and geometric transformations, such as compression, resizing, cropping, color changes, and rotation, as well as adaptive white-box and black-box adversarial attacks that modify an image to cross the matching boundary~\citep{jain2022adversarial,struppek2022learning,prokos2023squint,zhang2025atkscopes}. A low evasion or attack success rate indicates strong robustness against the tested transformations and attacks. However, because only a finite set of modifications is evaluated, empirical robustness does not guarantee resistance to stronger, unseen, or future attacks.

Certified robustness provides a complementary worst-case guarantee. It guarantees that no additive adversarial perturbation within a certified $\ell_2$ radius can change a near-duplicate match into a non-match. CertPHash~\citep{yang2025certphash} takes an important step in this direction by training perceptual hash models with specialized robustness objectives and formally verifying their resistance to bounded perturbations. However, this requires constructing and robustly training a dedicated hash model, making it difficult to apply directly to already deployed hashes. This motivates our key research question: \emph{can we improve both the empirical and certified robustness of existing deep perceptual hashes without retraining or modifying their underlying models?}

To answer this question, we propose \name, a plug-in defense that strengthens existing deep perceptual hashes, as depicted in Figure~\ref{fig:overview}. Consider an online platform where a content owner first publishes an image, which the platform then stores as a reference for future matching. Later, another user may download, modify, and re-upload that image -- for example, to evade harmful-content detection or copyright enforcement. The platform must therefore determine whether the new upload is still a near-duplicate of the original reference. Before an image is published, \name applies \emph{publication-time hardening} by adding an optimized \textbf{\textit{imperceptible}} perturbation, producing a hardened reference that is more resistant to future evasion attempts. Once published, only this hardened image is accessible to the adversary and can serve as the basis for a modified re-upload. When such a query is submitted, the platform applies \emph{matching-time randomized smoothing}, repeatedly comparing slightly perturbed versions of the reference-query pair and aggregating their decisions rather than relying on a single fragile comparison.

These two mechanisms form the complementary ``shields'' of \name. Randomized smoothing makes the matching decision more stable by combining multiple slightly noisy comparisons. However, when a reference-query pair lies close to the smoothed matching boundary, these comparisons may disagree, with some predicting a match and others a non-match. Publication-time hardening addresses this uncertainty by optimizing the reference specifically for the smoothed matcher. The optimization keeps potentially modified near-duplicates on the matching side and farther from the boundary, resulting in more consistent match decisions under randomized perturbations.

We evaluate \name across eight deep perceptual hashes and three image datasets under adaptive white-box and black-box adversarial attacks, common image transformations, hash collisions, and certified robustness. Across all hashes, \name reduces the average white-box attack success rate from 98.6\% to 11.8\% and the average black-box attack success rate from 20.6\% to 1.3\%, with several hashes achieving zero or near-zero attack success. At the same time, \name largely preserves the original collision behavior and maintains robustness to common image transformations, although some hashes exhibit a trade-off between adversarial robustness and transformation robustness. Beyond empirical evaluation, \name provides certified $\ell_2$ radii of approximately 0.3 across all evaluated hashes, including hashes that do not natively support certification. Overall, these results demonstrate that \name can substantially strengthen both the certified and empirical robustness of diverse existing deep perceptual hashes without retraining or modifying their underlying models.

\section{Related Work}

\paragraph{Near-Duplicate Image Matching and Perceptual Hashing.}
Perceptual hashing maps visually similar images to nearby compact representations for efficient near-duplicate image matching at scale. Classical methods such as Meta's PDQ~\citep{dalins2019pdq} and Microsoft's PhotoDNA~\citep{microsoft_photodna} rely on hand-crafted image transforms, while many newer approaches use learned visual representations~\citep{singhidinohash,he2024hybridhash,zhao2025kalahash,xiong2021perceptual,farid2021overview,ofcom2022perceptual}. These methods are designed to remain stable under common image transformations, but such robustness does not necessarily extend to adaptive adversarial attacks.

\paragraph{Empirical Robustness and Adversarial Evasion Attacks.}
Prior work has exposed the vulnerability of perceptual hashes to adaptive adversarial attacks that evade near-duplicate matching while preserving visual fidelity~\citep{hao2021s,bhatia2022exploiting,jain2022adversarial,struppek2022learning,prokos2023squint,zhang2025atkscopes,sun2025pgdattack}. These studies include gradient-based attacks in both white-box and black-box settings, and multi-resolution perturbations that transfer across perceptual hash systems. While these evaluations are crucial for assessing empirical robustness, and prior work has sought to strengthen defenses against some of these attacks~\citep{singh2021new,madden2024robustness,li2024perceptual}, robustness against a finite set of existing attacks does not necessarily extend to stronger or previously unseen attacks.

\paragraph{Certified Robustness for Perceptual Hashing.}
Certified robustness provides a formal guarantee that no adversarial perturbation within a bounded region can cause successful evasion. The closest work is CertPHash~\citep{yang2025certphash}, which combines robust training with neural-network verification to provide provable guarantees against evasion attacks. However, this requires constructing and robustly training a dedicated hash model, making it difficult to apply directly to already deployed hashes. In contrast, \name strengthens existing deep perceptual hashes as a plug-in defense, without retraining or modifying the underlying hash model.

\paragraph{Randomized Smoothing and Proactive Image Optimization.}
Randomized smoothing provides certified $\ell_2$ robustness by aggregating model predictions under Gaussian perturbations~\citep{cohen2019certified}. \name extends this principle beyond conventional single-image classification: it applies smoothing directly to the match/non-match decision of an existing perceptual hash, providing certified robustness for the underlying hash model. \name also proactively optimizes each image at publication time to increase its distance from the smoothed matching boundary. Unlike Active Image Indexing~\citep{fernandez2022active}, which improves retrieval under approximate indexing and common transformations, \name targets adversarial evasion. Together, these mechanisms improve empirical and certified robustness while retaining compatibility with existing perceptual hashes.

\section{Problem Formulation}
\label{sec:problem}

\subsection{Deep Perceptual Hashing}

A deep perceptual hash maps an image to a compact representation such that visually similar images have similar hashes.
Let $\mathcal{X}$ denote the image space and
\begin{equation}
f:\mathcal{X}\rightarrow\mathcal{H}
\end{equation}
denote a deep perceptual hash, where $\mathcal{H}$ is the hash representation space.
Depending on the hashing scheme, $f(x)$ may be a binary hash or a continuous embedding.

Deep perceptual hashes are commonly used for near-duplicate image matching.
Given a reference image $r$, a query image $q$, a similarity function $s$, and a matching threshold $\tau$, the hash matcher is defined as
\begin{equation}
M_f(r,q)
=
\begin{cases}
1 \quad \text{(match)}, & s\bigl(f(r),f(q)\bigr) \geq \tau,\\
0 \quad \text{(non-match)}, & \text{otherwise}.
\end{cases}
\label{eq:base-matcher}
\end{equation}

\subsection{Threat Model}

\paragraph{Attacker's goal.}
We focus on \emph{evasion attacks} against deep perceptual hashes.
The attacker starts from an image $q$ that should match a registered reference $r$ and modifies it into $q_{\mathrm{adv}}$ with the goal of changing its perceptual hash sufficiently to evade matching.
Formally, the attacker produces
\begin{equation}
q_{\mathrm{adv}}=\Pi_{\mathcal{X}}(q+\delta_a),
\qquad
\|\delta_a\|_2 \leq \epsilon_a,
\label{eq:adversarial-query}
\end{equation}
where $\delta_a$ is the adversarial perturbation, $\epsilon_a$ is the attack budget, and $\Pi_{\mathcal{X}}$ clips the perturbed image to the valid image domain.
An evasion attack succeeds when the original query is correctly matched to its reference, but the attacked query is not:
\begin{equation}
M_f(r,q)=1,
\qquad
M_f(r,q_{\mathrm{adv}})=0.
\label{eq:evasion-objective}
\end{equation}
Intuitively, the attack attempts to move the hash of the query from the matching region to the non-matching region while keeping the image visually similar to its original content. This threat is particularly relevant to content moderation, where an adversary may slightly modify previously detected content and re-upload it to bypass perceptual-hash-based detection. We focus on evasion attacks rather than targeted collision attacks. However, we also measure hash collision rates to ensure that improving robustness against evasion attacks does not cause unrelated images to become incorrectly matched.

\paragraph{Attacker's capabilities.}
We consider both \emph{white-box} and \emph{black-box} attackers.
A white-box attacker has full knowledge of the deep perceptual hash $f$, the similarity function $s$, the matching threshold $\tau$, and the defense.
The attacker can therefore use this information to optimize an adversarial image specifically against the protected hash.
A black-box attacker does not have access to the internal parameters of the target hash model.
The attacker instead constructs adversarial images without directly accessing the target model's gradients, for example through queries or transferable attacks.
We assume that the attacker cannot modify the platform's underlying hash model, matching threshold, or defense mechanism. The attacker controls only the image subsequently submitted for matching. Any randomness introduced by \name during matching is freshly sampled by the defender and cannot be controlled by the attacker.

\subsection{Robustness Goals}

We aim to improve two complementary properties of an existing deep perceptual hash: \emph{empirical robustness} and \emph{certified robustness}.

\paragraph{Empirical robustness.}
Empirical robustness measures whether a deep perceptual hash continues to produce the intended matching decision under concrete image modifications. We evaluate empirical robustness under two types of modifications.
First, we consider adaptive white-box and black-box adversarial attacks that deliberately attempt to move the query across the hash matching boundary. We measure their effectiveness using the \emph{attack success rate}, where a lower attack success rate indicates stronger robustness.
Second, we consider common image transformations.
Let
\begin{equation}
q_{\mathrm{tran}}=T(q),
\qquad
T\in\mathcal{T},
\label{eq:image-transformation}
\end{equation}
where $\mathcal{T}$ includes photometric and geometric transformations such as compression, resizing, cropping, color changes, blur, and rotation.
A robust perceptual hash should continue to produce a matching representation after these transformations.
Both evaluations are empirical: they demonstrate robustness against the attacks and transformations that are explicitly tested, but they do not guarantee robustness against every possible or future attack.

\paragraph{Certified robustness.}
Certified robustness provides a formal guarantee that a query cannot evade the protected deep perceptual hash within a specified perturbation region.
Let
\begin{equation}
S_f(r,q)\in\{0,1\}
\end{equation}
denote the protected hash matching decision produced by \name using an existing deep perceptual hash $f$.
For a reference-query pair classified as a match, we compute a certified $\ell_2$ radius $R_2$ such that
\begin{equation}
S_f(r,q+\delta)=1,
\qquad
\forall\,\delta \text{ such that } \|\delta\|_2 < R_2.
\label{eq:certified-goal}
\end{equation}
Thus, any adversarial perturbation to the query within the certified $\ell_2$ radius is formally guaranteed to preserve the near-duplicate match.
The certified radius $R_2$ is a provable lower bound on the robustness of the protected deep perceptual hash.
The hash may remain robust beyond this radius, but such robustness is not formally guaranteed by the certificate. 
Our certificate is \emph{query-only}: the registered reference is treated as fixed, while the attacker modifies the subsequently submitted query. The certificate does not certify arbitrary photometric or geometric transformations, which we evaluate empirically instead.

\paragraph{Objective.}
Given an existing deep perceptual hash $f$, our objective is to improve its robustness without retraining or modifying $f$ itself.
Specifically, we seek to (1) reduce attack success rates under adaptive white-box and black-box adversarial attacks, as well as image transformations, and (2) provide a large certified $\ell_2$ radius within which adversarial evasion is provably impossible.

\section{Methodology}
\label{sec:method}

\subsection{Overview}

\name strengthens an existing deep perceptual hash without retraining or modifying the underlying hash model. As depicted in Figure~\ref{fig:overview}, it introduces two complementary mechanisms at different stages of the image lifecycle:
\emph{publication-time hardening} and \emph{matching-time randomized smoothing}.
Consider an image $r$ that is published online and registered as a reference for detecting future near-duplicates.
Before publication, \name slightly modifies $r$ with an optimized imperceptible perturbation, producing a hardened image $r^\star$.
This modification is performed only once.
Later, an adversary may obtain the published image, modify it, and re-upload the resulting query $q$ in an attempt to evade near-duplicate matching.
Instead of relying on a single hash comparison, \name performs randomized evaluations around the reference-query pair and aggregates their match decisions.
The two mechanisms address complementary sources of vulnerability. Publication-time hardening places the reference in a more robust region of the hash matching space, making subsequent near-duplicates harder to push across the smoothed matching boundary. Matching-time randomized smoothing reduces the sensitivity of the final match/non-match decision to small adversarial perturbations and enables a formal $\ell_2$ robustness certificate.

\subsection{Publication-Time Hardening}
\label{sec:hardening}

Reference images can vary substantially in their robustness against evasion attacks. For some, only a small perturbation to a near-duplicate query is needed to cross the hash matching boundary, even under randomized smoothing. We therefore proactively optimize the reference image before it is published.

Given an original image $r$, we construct
\begin{equation}
r^\star
=
\Pi_{\mathcal{X}}(r+\delta_h),
\label{eq:publication-hardening}
\end{equation}
where $\delta_h$ is constrained to remain imperceptible and
$\Pi_{\mathcal{X}}$ clips the result to the valid image domain.

Our objective is to ensure that, after hardening, benign variants of an image match its reference with higher confidence, while unrelated images remain non-matches. We therefore optimize $\delta_h$ to enlarge the smoothed matching margin for positive pairs, namely a reference image and its benign variant, while preserving separation from negative pairs, namely the reference and unrelated images:
\begin{equation}
\delta_h^\star
=
\arg\max_{\delta_h\in\mathcal{C}_h}
\left[
\mathcal{J}_{\mathrm{pos}}(\delta_h)
-
\lambda_{\mathrm{neg}}
\mathcal{J}_{\mathrm{neg}}(\delta_h)
\right],
\label{eq:hardening-main}
\end{equation}
where $\mathcal{C}_h$ constrains the modification to be imperceptible,
$\mathcal{J}_{\mathrm{pos}}$ encourages confident matching with near-duplicates,
and $\mathcal{J}_{\mathrm{neg}}$ prevents unrelated images from becoming matches.
Importantly, hardening changes only the image being published.
The underlying deep perceptual hash remains unchanged.
Once generated, $r^\star$ is published and registered as the reference for subsequent matching.

Thus, an attacker who downloads $r^\star$ sees the hardened image itself. The attacker may still modify this image and attempt to re-upload it, but now starts from a reference that has been optimized to be less vulnerable to evasion attacks.

\subsection{Matching-Time Randomized Smoothing}
\label{sec:smoothing}

Publication-time hardening moves the reference image away from locally vulnerable regions, but a conventional deep perceptual hash still relies on a single deterministic comparison, making its decision potentially fragile and unreliable. A small perturbation may therefore shift the query representation across the matching boundary. Beyond increasing the reference image’s minimum distance to this boundary, \name smooths the decision rule by aggregating predictions over Gaussian-perturbed image pairs. This reduces sensitivity to local boundary irregularities and provides a statistically reliable decision that remains unchanged within a certifiable perturbation radius.

For a hardened reference $r^\star$ and query $q$, we independently sample Gaussian perturbations
\begin{equation}
\eta_r,\eta_q
\sim
\mathcal{N}(0,\sigma^2 I),
\end{equation}
and repeatedly evaluate the original hash matcher on
\begin{equation}
\left(
\Pi_{\mathcal{X}}(r^\star+\eta_r),
\Pi_{\mathcal{X}}(q+\eta_q)
\right).
\end{equation}
For each match outcome $c\in\{0,1\}$, define
\begin{equation}
p_c(r^\star,q)
=
\Pr
\left[
M_f\!\left(
\Pi_{\mathcal{X}}(r^\star+\eta_r),
\Pi_{\mathcal{X}}(q+\eta_q)
\right)
=c
\right].
\label{eq:smoothed-prob-main}
\end{equation}
The smoothed matcher returns the most probable decision:
\begin{equation}
G_\sigma(r^\star,q)
=
\arg\max_{c\in\{0,1\}}
p_c(r^\star,q).
\label{eq:smoothed-matcher-main}
\end{equation}
Intuitively, a query is no longer considered a match because of one potentially fragile hash comparison.
Instead, it must consistently remain a match across a neighborhood of randomized comparisons.
This creates a smoother and more stable match/non-match decision around the reference-query pair. In practice, we estimate $p_c$ using Monte Carlo sampling and compute a statistically valid lower confidence bound $\underline{p}_A$ for the probability of the match/non-match decision. The complete estimation and confidence-bound procedure is provided in the supplementary material.

\subsection{Certified Robustness}
\label{sec:certification}

Randomized smoothing converts the smoothed match probability into a formal robustness certificate, while hardening the reference with an optimized imperceptible perturbation increases the certified radius achievable by smoothing alone.
Suppose the predicted class has a valid lower confidence bound
$\underline{p}_A>1/2$.
We certify an $\ell_2$ radius
\begin{equation}
R_2
=
\sigma \Phi^{-1}(\underline{p}_A),
\label{eq:certified-radius-main}
\end{equation}
where $\Phi^{-1}$ is the inverse CDF of the standard normal distribution.
Our threat model fixes the registered reference $r^\star$ and allows the attacker to modify only the subsequently submitted query.
For a reference-query pair classified as a match, with confidence at least $1-\alpha$,
\begin{equation}
G_\sigma(r^\star,q+\delta)
=
G_\sigma(r^\star,q)
\qquad
\forall\,\|\delta\|_2<R_2.
\label{eq:query-guarantee-main}
\end{equation}
Therefore, if $G_\sigma(r^\star,q)=1$, no adversarial perturbation within the certified $\ell_2$ radius can change the decision from \emph{match} to \emph{non-match}.
The guarantee holds for every additive perturbation inside the certified region, regardless of how the attacker constructs it.

Intuitively, randomized smoothing makes the match decision stable over a neighborhood around $(r^\star,q)$.
If the match class remains sufficiently dominant under Gaussian perturbations, then a bounded shift of the query cannot change the majority decision.
A larger $\underline{p}_A$ therefore yields a larger certified radius.
Publication-time hardening complements this process by improving the local matching margin, which can increase the resulting certified robustness.
The full theorem, proof, confidence-bound construction, and certification algorithm are provided in the supplementary material.

\begin{table*}[tb]
\centering
\small
\setlength{\tabcolsep}{2.4pt}
\renewcommand{\arraystretch}{1.10}

\begin{tabular*}{\textwidth}{
@{\extracolsep{\fill}}
l
cc
cc
cc
cc
cc
}
\toprule

\multirow{3}{*}{\textbf{Hash}}
& \multicolumn{6}{c}{\textbf{Empirical Robustness} ($\%$, $\downarrow$)}
& \multicolumn{2}{c}{\textbf{Certified Robustness} ($\uparrow$)}
& \multicolumn{2}{c}{\textbf{Utility} ($\%$, $\downarrow$)}
\\

\cmidrule(lr){2-7}
\cmidrule(lr){8-9}
\cmidrule(lr){10-11}

& \multicolumn{2}{c}{White-box ASR}
& \multicolumn{2}{c}{Black-box ASR}
& \multicolumn{2}{c}{Transformation ASR}
& \multicolumn{2}{c}{$\overline{R}_2$}
& \multicolumn{2}{c}{Collision Rate}
\\

\cmidrule(lr){2-3}
\cmidrule(lr){4-5}
\cmidrule(lr){6-7}
\cmidrule(lr){8-9}
\cmidrule(lr){10-11}

& Original & Ours
& Original & Ours
& Original & Ours
& Original & Ours
& Original & Ours
\\

\midrule

NeuralHash
& 100.0 & \textbf{20.6}
& 35.2  & \textbf{4.4}
& \textbf{0.0} & 7.7
& -- & \textbf{0.2987}
& \textbf{0.0} & \textbf{0.0}
\\

DINOHash
& 100.0 & \textbf{14.0}
& 15.6  & \textbf{0.0}
& \textbf{0.5} & 2.8
& -- & \textbf{0.2993}
& \textbf{0.0} & \textbf{0.0}
\\

HybridHash
& 100.0 & \textbf{1.0}
& 22.0  & \textbf{0.0}
& \textbf{0.0} & 0.5
& -- & \textbf{0.2993}
& \textbf{0.1} & 0.3
\\

SSCD
& 100.0 & \textbf{3.3}
& 73.1  & \textbf{5.9}
& \textbf{0.5} & 23.6
& -- & \textbf{0.2993}
& \textbf{0.0} & \textbf{0.0}
\\

SimDINO
& 100.0 & \textbf{43.7}
& 16.6  & \textbf{0.0}
& \textbf{0.1} & 8.9
& -- & \textbf{0.2993}
& \textbf{0.0} & \textbf{0.0}
\\

ViT2Hash
& 100.0 & \textbf{12.0}
& 2.3   & \textbf{0.0}
& \textbf{0.1} & 3.9
& -- & \textbf{0.2993}
& 0.1 & \textbf{0.0}
\\

C-PDNA
& 94.2 & \textbf{0.0}
& \textbf{0.0} & \textbf{0.0}
& 0.8 & \textbf{0.3}
& 0.2968 & \textbf{0.2993}
& 5.9 & \textbf{5.7}
\\

C-PDQ
& 94.7 & \textbf{0.0}
& \textbf{0.0} & \textbf{0.0}
& \textbf{0.0} & \textbf{0.0}
& 0.1282 & \textbf{0.2993}
& \textbf{100.0} & \textbf{100.0}
\\

\bottomrule
\end{tabular*}
\caption{
Overall evaluation across eight perceptual hashes.
\emph{Empirical robustness} is evaluated by white-box and black-box attack success rate (ASR) and transformation ASR (all in \%, $\downarrow$). White-box and black-box ASR are averaged over attack budgets 40, 90, and 180, while transformation ASR is averaged across tested transformations such as compression, cropping and rotation.
\emph{Certified robustness} is reported as the mean query-only certified $\ell_2$ radius $\overline{R}_2$ ($\uparrow$).
\emph{Utility} is measured by the collision rate (\%, $\downarrow$) on unrelated image pairs.
A dash indicates that the original hash does not provide a directly comparable native certificate.
}
\label{tab:overall-summary}
\end{table*}

\section{Evaluation}
\label{sec:evaluation}

We evaluate whether \name can strengthen existing deep perceptual hashes along the two robustness goals:
(1)~\emph{empirical robustness} against concrete adversarial attacks and image transformations, and (2)~\emph{certified robustness} against $\ell_2$-norm-bounded adversarial perturbations.
We further examine whether the robustness gains generalize across hash families and whether they come at the cost of increased collisions or degraded visual qualities.

\subsection{Experimental Setup}

\paragraph{Hash models.}
We evaluate eight perceptual-hash systems from diverse model families: NeuralHash~\citep{struppek2022learning}, DINOHash~\citep{singhidinohash}, HybridHash~\citep{he2024hybridhash}, SSCD~\citep{pizzi2022sscd}, SimDINO~\citep{moummad2026hashing}, ViT2Hash~\citep{gong2201vit2hash}, C-PDNA, and C-PDQ~\citep{yang2025certphash}. Together, they cover convolutional and transformer-based deep hashing models, learned image descriptors adapted for perceptual matching, and perceptual-hash models with certified robustness guarantees.

\paragraph{Datasets.}
We evaluate each hash on ImageNet~\citep{imagenet}, MS-COCO~\citep{lin2014microsoft}, and Stable-Diffusion images~\citep{rombach2022highresolution}. ImageNet and MS-COCO cover diverse natural images, while Stable-Diffusion provides generated images relevant to provenance and online content-sharing applications.

\paragraph{Compared variants.}
We compare three variants.
\emph{Original} is the original perceptual hash and its native matching rule.
\emph{Smoothing} applies only matching-time randomized smoothing.
\emph{Ours} combines matching-time randomized smoothing with publication-time hardening.
This comparison allows us to highlight the effect of publication-time hardening beyond randomized smoothing alone.

\paragraph{Metrics.}
For empirical adversarial robustness, we report \textit{attack success rate} (ASR) under adaptive white-box and black-box adversarial attacks at three increasing attack budgets in terms of $\ell_2$ norm: 40, 90, and 180. Lower ASR indicates stronger robustness. For common photometric and geometric image transformations, we report the transformation ASR as the fraction of transformed near-duplicate queries that are no longer detected as matches. We also measure \textit{collision rates} on unrelated images to verify that robustness is not achieved simply by making the matcher more permissive.
For certified robustness, we report the mean query-only certified $\ell_2$ radius $\overline{R}_2$.
A larger radius means that a larger region around the query is formally guaranteed to preserve the match/non-match decision.
Detailed attack settings, transformation parameters, per-budget results, and certification statistics are provided in the supplementary material.

\subsection{Overall Robustness}

Table~\ref{tab:overall-summary} summarizes the results across all eight perceptual hashes.
The main observation is consistent across model families: \name substantially improves resistance to adversarial evasion attacks while simultaneously providing a formal robustness certificate.

\paragraph{\name substantially reduces success rates of adversarial evasion attacks.}
The original hashes are highly vulnerable to adaptive white-box attacks: averaged across all hashes and attack budgets, their ASR is 98.6\%.
\name reduces this to 11.8\%, an absolute reduction of 86.8 percentage points.
The improvement is observed across all eight hashes, despite their substantially different representations and matching rules.
For HybridHash and the two CertPHash variants (i.e., C-PDNA and C-PDQ), the average white-box ASR is at most 1\% after applying \name.
The same trend holds under black-box attacks. Average ASR decreases from 20.6\% for the original hashes to 1.3\% with \name. Six of the eight hashes achieve zero observed black-box ASR across all three tested budgets; only NeuralHash and SSCD exhibit successful attacks at the largest budget. These results support our claim that \name provides a broadly applicable robustness enhancement rather than relying on properties of a particular hash architecture.

\paragraph{Robustness remains strong at practical attack budgets.}
\name is particularly effective under small to moderate attack budgets. At budget 40, the average white-box ASR is only 0.09\% across all eight hashes, with seven achieving 0\% ASR. At budget 90, the average remains low at 4.6\%. At the largest budget of 180, the average ASR increases to 30.8\%, with NeuralHash and SimDINO becoming more vulnerable.
Importantly, stronger attacks introduce substantially greater visual distortion. As the attack budget increases from 40 to 180, Structural Similarity Index Measure (SSIM) decreases sharply while Learned Perceptual Image Patch Similarity (LPIPS) and Fréchet Inception Distance (FID) increase across all hashes, indicating that successful evasion at larger budgets comes at the cost of increasingly degraded image quality. Thus, while \name does not prevent evasion under arbitrarily large perturbations, the attacker must trade visual fidelity for attack success. Detailed per-budget results are provided in the supplementary material.

\paragraph{\name balances adversarial and transformation robustness.}
We further evaluate whether improving resistance to adversarial evasion comes at the cost of robustness to common photometric and geometric transformations. For seven of the eight hashes, \name keeps the transformation evasion rate below 10\%, and the degradation is small for DINOHash, HybridHash, ViT2Hash, C-PDNA, and C-PDQ. In particular, C-PDNA and C-PDQ show little to no loss in transformation robustness while achieving strong adversarial robustness with \name.
The trade-off becomes more apparent for hashes that require stronger intervention to resist adversarial evasion. SSCD, for example, achieves a large reduction in adversarial ASR but its transformation evasion rate increases from 0.5\% to 23.6\%, mainly under cropping and rotation; SimDINO shows a smaller version of the same effect, increasing from 0.1\% to 8.9\%. These results suggest that adversarial and transformation robustness must be balanced in a model-dependent manner: when the underlying hash is already easier to defend, \name can substantially improve adversarial robustness with little compromise in transformation robustness, whereas more vulnerable hashes may require stronger hardening and incur a larger trade-off.

\paragraph{From empirical resistance to formal guarantees.}
Unlike empirical attack evaluation, certification does not depend on a particular attack algorithm. Across the eight evaluated hashes, \name achieves a mean certified radius of approximately $0.3$. The empirical and certified results capture complementary notions of robustness. Empirically, the hashes often withstand perturbations much larger than their certified $\ell_2$ radii, showing that current white-box and black-box attacks may fail well before reaching the true worst-case vulnerability of the matcher. Certification is intentionally more conservative: it provides a formal lower bound within which \emph{no} possible attack can change the match decision, regardless of how the perturbation is constructed. This is important because existing attacks may not yet exploit the most vulnerable directions of the hash representation and matching rule, while stronger future attacks potentially could. Improving certified robustness therefore complements empirical evaluation by strengthening the guaranteed worst-case behavior of the system. \name provides such guarantees for hashes that do not natively support certification, while also improving or maintaining the certified bounds of CertPHash variants (i.e., C-PDNA and C-PDQ).

\begin{table}[tb]
\centering
\small
\setlength{\tabcolsep}{4pt}
\renewcommand{\arraystretch}{1.08}

\begin{tabular}{lccc}
\toprule
\textbf{Variant}
& White-box $\downarrow$
& Black-box $\downarrow$
& $\boldsymbol{\overline{R}_2}$ $\uparrow$ \\
\midrule
Original
& 98.6
& 20.6
& -- \\

+ Smoothing
& 43.8
& 18.0
& 0.252 \\

+ Hardening (Ours)
& \textbf{11.8}
& \textbf{1.3}
& \textbf{0.299} \\
\bottomrule
\end{tabular}
\caption{
Ablation of \name, averaged across all eight perceptual hashes.
ASR is averaged over attack budgets 40, 90, and 180.
Lower ASR and higher $\overline{R}_2$ are better.
}
\label{tab:ablation-summary}
\end{table}

\subsection{Why Are Both ``Shields'' Necessary?}

To understand the contribution of publication-time hardening, we compare full \name against matching-time randomized smoothing alone. Table~\ref{tab:ablation-summary} aggregates the results across all eight hashes and attack budgets.
Matching-time smoothing alone already improves white-box robustness, reducing average ASR from 98.6\% to 43.8\%.
However, this protection is highly dependent on the reference image and underlying hash. For example, smoothing alone remains vulnerable at larger budgets for NeuralHash, SSCD, and SimDINO.
Publication-time hardening addresses these fragile cases. When combined with smoothing, average white-box ASR further decreases from 43.8\% to 11.8\%, while black-box ASR decreases from 18.0\% to 1.3\%. The mean certified radius also increases from $0.252$ to $0.299$.
These results support the intuition behind \name's two-stage design.
Randomized smoothing makes the matching decision locally more stable and enables certification, while publication-time hardening improves the operating point of each individual reference before an attack occurs. Neither mechanism alone explains the full robustness gain.

\begin{table}[tb]
\centering
\small
\setlength{\tabcolsep}{4.2pt}
\renewcommand{\arraystretch}{1.04}
\begin{tabular}{cccc}
\toprule
Hash & SSIM $\uparrow$ & LPIPS $\downarrow$ & FID $\downarrow$ \\
\midrule
DINO    & 0.700 & 0.271 & 57.7 \\
Hybrid  & 0.729 & 0.233 & 47.1 \\
C-PDNA  & 0.922 & 0.079 & 13.2 \\
C-PDQ   & \textbf{0.987} & \textbf{0.029} & \textbf{4.6} \\
\bottomrule
\end{tabular}
\caption{Visual similarity between original and hardened images for four representative hashes, averaged over ImageNet, MS-COCO, and Stable-Diffusion. Higher SSIM and lower LPIPS/FID indicate better visual preservation.}
\label{tab:main-hardening-quality}
\end{table}

\subsection{Hash Collision and Visual Fidelity}

Improving robustness against evasion should not make a perceptual hash less useful for its original purpose.
We therefore examine two potential side effects of \name: collisions between unrelated images, and visual changes introduced by publication-time hardening.

\paragraph{Collision behavior.}
A stronger defense should not make the matcher overly permissive and cause unrelated images to match. As shown in Table~\ref{tab:overall-summary}, for most hashes, \name leaves collision rates unchanged or near zero.
NeuralHash, DINOHash, SSCD, and SimDINO remain at 0\%, while ViT2Hash decreases from 0.1\% to 0\%. HybridHash increases slightly from 0.1\% to 0.3\%, and C-PDNA remains approximately unchanged.
C-PDQ is an important exception: it already exhibits a 100\% collision rate under the evaluated configuration, and \name does not correct this pre-existing limitation. This reflects the intended scope of our plug-in design: \name strengthens an existing hash against evasion, but does not attempt to repair fundamental utility limitations of the underlying hash itself.

\paragraph{Visual quality after publication-time hardening.}
Publication-time hardening modifies the image before it is released, so the resulting image should remain visually close to the original.
Table~\ref{tab:main-hardening-quality} reports SSIM, LPIPS, and FID between the original and hardened images for four representative hashes, averaged across the three datasets.
Higher SSIM and lower LPIPS and FID indicate better visual preservation.
The amount of change depends strongly on the underlying hash.
C-PDQ and C-PDNA preserve the original appearance particularly well, whereas DINOHash and HybridHash require larger modifications at the evaluated hardening strength.
Figure~\ref{fig:main-hardening-examples} illustrates this difference on the same source image:
C-PDQ is nearly unchanged, C-PDNA introduces only mild changes, while DINOHash and HybridHash exhibit more visible texture differences. This suggests another practical trade-off between robustness and visual fidelity. When preserving appearance is more important, the hardening strength can be reduced at the cost of a smaller robustness gain.

\begin{figure}[tb]
  \centering
  \setlength{\tabcolsep}{1pt}
  \renewcommand{\arraystretch}{1.0}

  \begin{tabular}{@{}lcccc@{}}
    & \scriptsize\textbf{DINO}
    & \scriptsize\textbf{Hybrid}
    & \scriptsize\textbf{C-PDNA}
    & \scriptsize\textbf{C-PDQ}
    \\

    \scriptsize\textbf{Original}
    & \includegraphics[width=0.185\columnwidth]{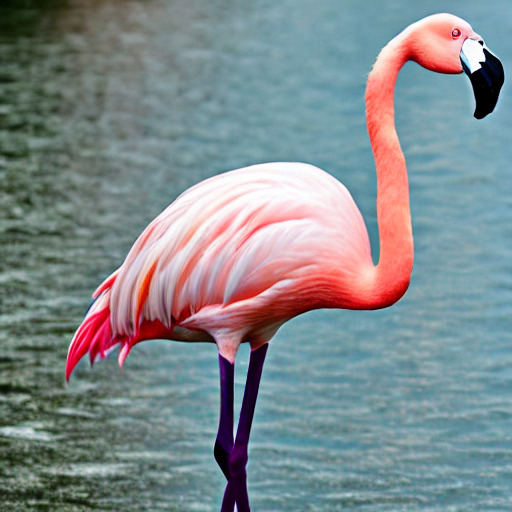}
    & \includegraphics[width=0.185\columnwidth]{Figures/hardening_examples/comparison_original.jpg}
    & \includegraphics[width=0.185\columnwidth]{Figures/hardening_examples/comparison_original.jpg}
    & \includegraphics[width=0.185\columnwidth]{Figures/hardening_examples/comparison_original.jpg}
    \\[-1pt]

    \scriptsize\textbf{Hardened}
    & \includegraphics[width=0.185\columnwidth]{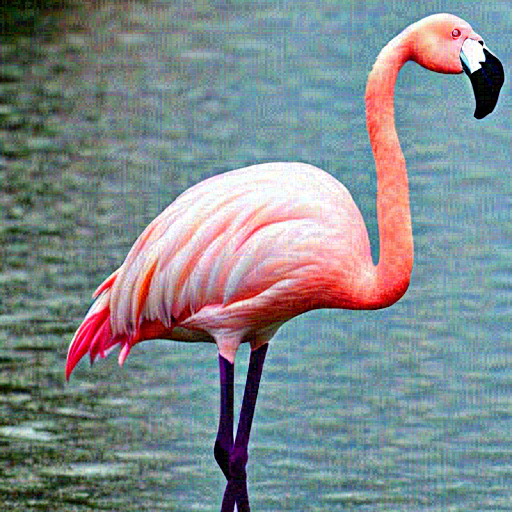}
    & \includegraphics[width=0.185\columnwidth]{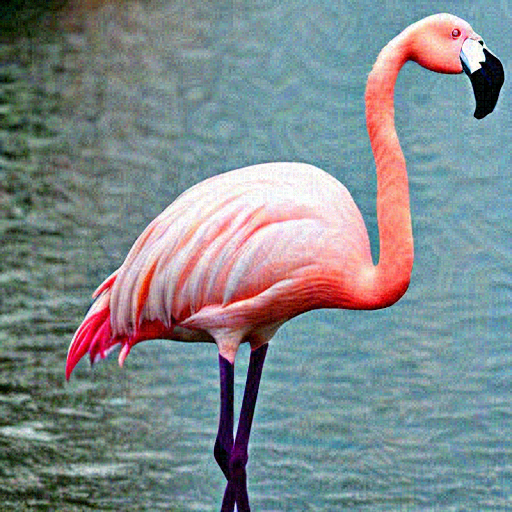}
    & \includegraphics[width=0.185\columnwidth]{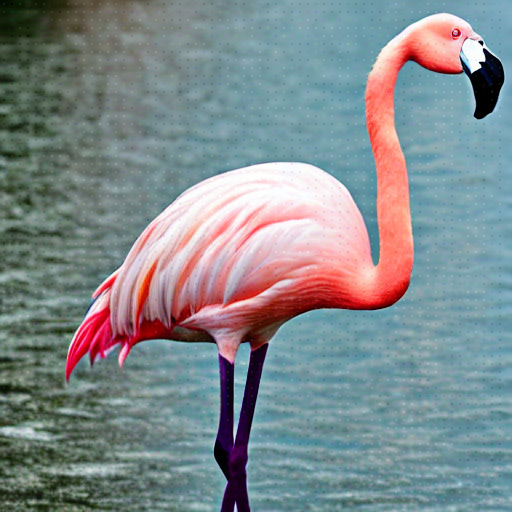}
    & \includegraphics[width=0.185\columnwidth]{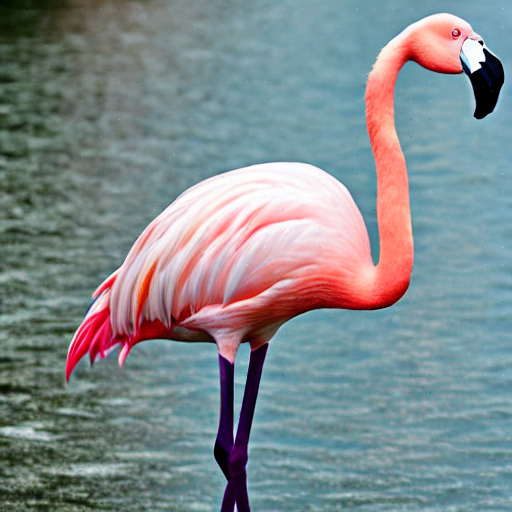}
  \end{tabular}

  \caption{
  Visual comparison of publication-time hardening on the same source image.
  Visual changes are model-dependent: C-PDQ and C-PDNA preserve the original appearance more closely than DINOHash and HybridHash.
  }
  \label{fig:main-hardening-examples}
\end{figure}

\section{Discussion and Conclusion}

\name shows that robust deep perceptual hashing does not necessarily require redesigning or retraining the underlying hash model. Instead, robustness can be added around an existing hash by jointly protecting the published reference and strengthening the matching decision. This suggests a broader design principle: \emph{robustness can be treated as a plug-in layer} that augments deployed visual matching systems without changing their core representations. At the same time, several open challenges remain. Future work could extend certification beyond additive $\ell_2$ perturbations to structured image transformations, reduce the computational cost of randomized smoothing, and study stronger adaptive or collision-based attacks. More broadly, we hope \name motivates new approaches for improving the robustness of existing perceptual hashes after deployment, making certified and adversarially robust matching easier to adopt in real-world systems.

\bibliography{aaai2027}

\appendix
\setcounter{secnumdepth}{1}

\renewcommand{\thefigure}{\Alph{section}.\arabic{figure}}
\renewcommand{\thetable}{\Alph{section}.\arabic{table}}
\renewcommand{\theequation}{\Alph{section}.\arabic{equation}}
\renewcommand{\thealgorithm}{\Alph{section}.\arabic{algorithm}}

\makeatletter
\@addtoreset{figure}{section}
\@addtoreset{table}{section}
\@addtoreset{equation}{section}
\@addtoreset{algorithm}{section}
\makeatother

\section{Detailed Methodology and Theoretical Guarantee}
\label{sec:supp-method}

This section provides the formal definitions, certification procedure,
optimization details, algorithms, and proof underlying \name.
The main paper presents the same method at a higher level for clarity.

\subsection{Formal Definition of the Smoothed Hash Matcher}

Let
\[
M_f:\mathcal{X}\times\mathcal{X}\rightarrow\{0,1\}
\]
denote the deterministic matching rule induced by an existing deep perceptual hash $f$.
For a hardened reference $r^\star$ and query $q$, define the concatenated pair
\begin{equation}
z
=
\begin{bmatrix}
r^\star\\
q
\end{bmatrix}
\in\mathbb{R}^{2n}.
\label{eq:supp-pair-input}
\end{equation}

We define the corresponding lifted matcher
\begin{equation}
\overline{M}_f(z)
=
M_f\!\left(
\Pi_{\mathcal{X}}(r^\star),
\Pi_{\mathcal{X}}(q)
\right),
\label{eq:supp-lifted-matcher}
\end{equation}
where deterministic preprocessing and clipping are treated as part of the base matcher.

At matching time, we add independent Gaussian noise to the reference and query:
\begin{equation}
\eta
=
\begin{bmatrix}
\eta_r\\
\eta_q
\end{bmatrix}
\sim
\mathcal{N}(0,\sigma^2I_{2n}),
\qquad
\eta_r,\eta_q
\stackrel{\mathrm{i.i.d.}}{\sim}
\mathcal{N}(0,\sigma^2I_n).
\label{eq:supp-pair-noise}
\end{equation}

For each class $c\in\{0,1\}$, define
\begin{equation}
p_c(z)
=
\Pr_{\eta}
\left[
\overline{M}_f(z+\eta)=c
\right].
\label{eq:supp-smoothed-probability}
\end{equation}

The ideal Gaussian-smoothed matcher is
\begin{equation}
G_\sigma(z)
=
\arg\max_{c\in\{0,1\}}p_c(z).
\label{eq:supp-smoothed-matcher}
\end{equation}

Because $M_f$ may be any deterministic deep perceptual hash matcher,
this construction does not require retraining or modifying the underlying hash model.

\subsection{Monte Carlo Estimation and Statistical Certification}

The probabilities in Equation~\eqref{eq:supp-smoothed-probability}
cannot generally be computed exactly.
We estimate them using Monte Carlo sampling.

Following standard randomized smoothing, we use independent samples for
class selection and certification.

First, draw $N_0$ independent samples
\[
\eta^{(i)}
\sim
\mathcal{N}(0,\sigma^2I_{2n}),
\qquad
i=1,\ldots,N_0,
\]
and select the most frequent class:
\begin{equation}
\hat{c}_A
=
\arg\max_{c\in\{0,1\}}
\sum_{i=1}^{N_0}
\mathbb{I}
\left[
\overline{M}_f(z+\eta^{(i)})=c
\right].
\label{eq:supp-class-selection}
\end{equation}

If the indicator notation $\mathbb{I}$ is not available in the document class,
the count can equivalently be written as the cardinality of the corresponding
sample set.

Next, draw an independent set of $N$ samples
\[
\tilde{\eta}^{(i)}
\sim
\mathcal{N}(0,\sigma^2I_{2n}),
\qquad
i=1,\ldots,N,
\]
and count
\begin{equation}
K_A
=
\left|
\left\{
i:
\overline{M}_f(z+\tilde{\eta}^{(i)})
=
\hat{c}_A
\right\}
\right|.
\label{eq:supp-cert-count}
\end{equation}

We compute a one-sided Clopper-Pearson lower confidence bound
$\underline{p}_A$ for the probability of $\hat{c}_A$:
\begin{equation}
\underline{p}_A
=
\begin{cases}
0,
&
K_A=0,
\\[3pt]
F^{-1}_{\mathrm{Beta}
(K_A,N-K_A+1)}(\alpha),
&
K_A>0,
\end{cases}
\label{eq:supp-clopper-pearson}
\end{equation}
where $F^{-1}_{\mathrm{Beta}(a,b)}$ denotes the inverse CDF of a
$\mathrm{Beta}(a,b)$ random variable.

With probability at least $1-\alpha$ over the certification samples,
\begin{equation}
p_{\hat{c}_A}(z)
\geq
\underline{p}_A.
\label{eq:supp-confidence-event}
\end{equation}

Certification is returned only when
\begin{equation}
\underline{p}_A>\frac{1}{2}.
\label{eq:supp-majority-condition}
\end{equation}

\subsection{Query-Only Certified $\ell_2$ Radius}

Since matching is binary,
\begin{equation}
p_{1-\hat{c}_A}(z)
=
1-p_{\hat{c}_A}(z).
\end{equation}

When $\underline{p}_A>1/2$, the Gaussian randomized-smoothing radius is
\begin{equation}
R_2(z)
=
\sigma\Phi^{-1}(\underline{p}_A).
\label{eq:supp-l2-radius}
\end{equation}

Although smoothing is defined over the reference-query pair, our attacker
perturbs only the query.
A query perturbation $\delta_q$ corresponds to the pairwise displacement
\begin{equation}
\Delta
=
\begin{bmatrix}
0\\
\delta_q
\end{bmatrix}.
\label{eq:supp-query-shift}
\end{equation}

Therefore,
\begin{equation}
\|\Delta\|_2
=
\|\delta_q\|_2.
\label{eq:supp-query-equality}
\end{equation}

The native pairwise $\ell_2$ certificate thus directly yields a query-only
$\ell_2$ certificate with the same radius $R_2$.

\subsection{Theoretical Guarantee}

\begin{theorem}[Query-only $\ell_2$ robustness]
\label{thm:supp-query-l2}
Let
$\overline{M}_f:\mathbb{R}^{2n}\rightarrow\{0,1\}$
be the deterministic lifted hash matcher in
Equation~\eqref{eq:supp-lifted-matcher},
and let $G_\sigma$ be its Gaussian-smoothed matcher in
Equation~\eqref{eq:supp-smoothed-matcher}.

Suppose $\hat{c}_A$ is the selected class and
$\underline{p}_A>1/2$ is a valid lower confidence bound satisfying
Equation~\eqref{eq:supp-confidence-event}.
Define
\begin{equation}
R_2
=
\sigma\Phi^{-1}(\underline{p}_A).
\end{equation}

Then, with probability at least $1-\alpha$ over the certification samples,
for every query-only adversarial perturbation
$\delta_q\in\mathbb{R}^{n}$ satisfying
\begin{equation}
\|\delta_q\|_2<R_2,
\end{equation}
the smoothed match decision is unchanged:
\begin{equation}
G_\sigma(r^\star,q+\delta_q)
=
G_\sigma(r^\star,q)
=
\hat{c}_A.
\label{eq:supp-theorem-result}
\end{equation}
\end{theorem}

\begin{proof}
Define the acceptance region of class $\hat{c}_A$ as
\begin{equation}
\mathcal{A}
=
\left\{
u\in\mathbb{R}^{2n}:
\overline{M}_f(u)=\hat{c}_A
\right\}.
\end{equation}

On the confidence event in
Equation~\eqref{eq:supp-confidence-event},
\begin{equation}
\Pr_{u\sim
\mathcal{N}(z,\sigma^2I_{2n})}
[u\in\mathcal{A}]
\geq
\underline{p}_A.
\label{eq:supp-base-prob}
\end{equation}

For a displacement $\Delta$, the Gaussian Neyman-Pearson argument underlying
randomized smoothing gives
\begin{equation}
\Pr_{u\sim
\mathcal{N}(z+\Delta,\sigma^2I_{2n})}
[u\in\mathcal{A}]
\geq
\Phi
\left(
\Phi^{-1}(\underline{p}_A)
-
\frac{\|\Delta\|_2}{\sigma}
\right).
\label{eq:supp-np-bound}
\end{equation}

For a query-only perturbation,
\[
\Delta
=
\begin{bmatrix}
0\\
\delta_q
\end{bmatrix},
\]
and therefore
\[
\|\Delta\|_2
=
\|\delta_q\|_2.
\]

If
\begin{equation}
\|\delta_q\|_2
<
\sigma\Phi^{-1}(\underline{p}_A)
=
R_2,
\end{equation}
then
\begin{equation}
\Phi
\left(
\Phi^{-1}(\underline{p}_A)
-
\frac{\|\delta_q\|_2}{\sigma}
\right)
>
\Phi(0)
=
\frac{1}{2}.
\end{equation}

Thus, class $\hat{c}_A$ remains the unique majority class under Gaussian
smoothing after the perturbation.
Therefore,
\begin{equation}
G_\sigma(r^\star,q+\delta_q)
=
\hat{c}_A.
\end{equation}
\end{proof}

\paragraph{Interpretation.}
When $\hat{c}_A=1$, Theorem~\ref{thm:supp-query-l2} guarantees that every
adversarial modification within the certified $\ell_2$ ball remains a match.
The guarantee is attack-independent: it holds regardless of the attack
algorithm, optimization procedure, or attacker knowledge.

The radius $R_2$ is a certified lower bound.
The matcher may remain robust to perturbations outside this radius, but such
robustness is not guaranteed by the certificate.

The theorem does not certify arbitrary geometric or photometric transformations.
Such transformations are evaluated empirically in our experiments.

\subsection{Publication-Time Hardening}

Randomized smoothing provides certification around the deployed
reference-query pair, but the resulting radius depends on the smoothed
winning-class probability.
Some reference images naturally produce weaker match probabilities and
therefore smaller certified regions.

Publication-time hardening proactively modifies the reference to improve this
local matching geometry before the image is released.

Given an original image $r$, positive near-duplicate samples
$\mathcal{P}$, and negative samples $\mathcal{N}$, we optimize
\begin{equation}
\max_{\delta_h\in\mathcal{C}_h}
\;
\mathcal{J}_{\mathrm{pos}}(\delta_h)
-
\lambda_{\mathrm{neg}}
\mathcal{J}_{\mathrm{neg}}(\delta_h)
-
\lambda_{\mathrm{dist}}
D\!\left(
r,
\Pi_{\mathcal{X}}(r+\delta_h)
\right),
\label{eq:supp-hardening-objective}
\end{equation}
where
$\mathcal{J}_{\mathrm{pos}}$ encourages high smoothed match confidence for positive
pairs,
$\mathcal{J}_{\mathrm{neg}}$ penalizes matches with unrelated images,
and $D$ constrains perceptual distortion.

\paragraph{Optimization details.}
All images are represented as $512\times512$ RGB tensors in $[0,1]$.
We use
\begin{equation}
\mathcal{C}_h
=
\left\{
\delta_h:
\|\delta_h\|_\infty\leq\frac{16}{255}
\right\},
\end{equation}
and perform 200 projected sign-gradient updates with step size $1/255$.
The positive objective uses 16 Gaussian expectation-over-transformation
samples per update at $\sigma=0.10$ and an inner five-step adversarial
optimization. The negative objective uses eight Gaussian samples per selected
negative. We use positive and negative hash-margin targets of four bits,
soft-Hamming sharpness $8$, $\lambda_{\mathrm{neg}}=1$, and an
average-squared-pixel distortion weight of $0.01$.
These settings are shared across the eight hash models and the resulting visual
change remains model-dependent.

The hardened image is
\begin{equation}
r^\star
=
\Pi_{\mathcal{X}}(r+\delta_h^\star).
\label{eq:supp-hardened-reference}
\end{equation}

Hardening itself does not constitute the certificate.
The final certificate is always computed from the original discrete matcher
$M_f$, the deployed hardened reference $r^\star$, and the randomized-smoothing
procedure described above.

Its role is instead to improve the operating point at which certification is
performed.
Increasing the smoothed winning probability $\underline{p}_A$ increases
\begin{equation}
R_2
=
\sigma\Phi^{-1}(\underline{p}_A),
\end{equation}
and therefore enlarges the certified region.

\subsection{Algorithms}
Algorithms~\ref{alg:supp-hardening} and~\ref{alg:supp-smoothing} summarize the two components of \name. Algorithm~\ref{alg:supp-hardening} presents the publication-time hardening procedure, which optimizes a bounded perturbation for each reference image using positive and negative pairs under Gaussian smoothing. Algorithm~\ref{alg:supp-smoothing} presents the matching-time procedure, which uses independent Monte Carlo samples to determine the smoothed match decision, estimate its confidence, and return a certified $\ell_2$ radius when the confidence condition is satisfied.

\begin{algorithm}[t]
\caption{Publication-Time Hardening}
\label{alg:supp-hardening}
\begin{algorithmic}[1]
\Require Original image $r$, positive set $\mathcal{P}$,
negative set $\mathcal{N}$, feasible set $\mathcal{C}_h$,
iterations $K$
\State Initialize $\delta_h\leftarrow0$
\For{$k=1$ to $K$}
    \State Sample positive queries, negative queries, and Gaussian noise
    \State Estimate differentiable surrogate smoothed match objectives
    \State Compute the gradient of
    Equation~\eqref{eq:supp-hardening-objective}
    with respect to $\delta_h$
    \State Update $\delta_h$ and project it onto $\mathcal{C}_h$
\EndFor
\State $r^\star\leftarrow
\Pi_{\mathcal{X}}(r+\delta_h)$
\State \Return $r^\star$
\end{algorithmic}
\end{algorithm}

\begin{algorithm}[t]
\caption{Matching-Time Smoothed Matching and Certification}
\label{alg:supp-smoothing}
\begin{algorithmic}[1]
\Require Original hash matcher $M_f$, hardened reference $r^\star$,
query $q$, noise scale $\sigma$, selection samples $N_0$,
certification samples $N$, confidence level $\alpha$

\State Draw $N_0$ independent Gaussian-noised reference-query pairs
\State Evaluate $M_f$ on each pair
\State Select the majority class $\hat{c}_A$

\State Draw an independent set of $N$ Gaussian-noised pairs
\State Count $K_A$, the number classified as $\hat{c}_A$
\State Compute the lower confidence bound $\underline{p}_A$

\If{$\underline{p}_A>1/2$}
    \State $R_2\leftarrow
    \sigma\Phi^{-1}(\underline{p}_A)$
    \State \Return $\hat{c}_A$ and certified radius $R_2$
\Else
    \State \Return abstention from certification
\EndIf
\end{algorithmic}
\end{algorithm}

\section{Additional Experimental Details and Results}
\label{sec:supp-evaluation}

This section provides the complete experimental protocol and detailed results
supporting the evaluation in the main paper. We evaluate \name on eight
perceptual-hash systems -- NeuralHash, DINOHash, HybridHash, SSCD, SimDINO,
ViT2Hash, C-PDNA, and C-PDQ -- across ImageNet, MS-COCO, and
Stable-Diffusion images.

We separately evaluate
(1) certified robustness to query-only additive adversarial perturbations,
(2) empirical robustness to adaptive white-box and black-box adversarial
evasion attacks,
(3) empirical robustness to common photometric and geometric image
transformations,
(4) collision behavior on unrelated images, and
(5) visual quality after publication-time hardening and adversarial attacks.

\subsection{Detailed Experimental Protocol}

\paragraph{Images, matching, and aggregation.}
All images are converted to RGB, resized to $512\times512$, represented in
$[0,1]$, and clipped to this range after modification. We use the matching
threshold $0.2$ for all eight hashes (i.e., bit error rate $\leq$ 20\%). Adaptive attacks and certification use
100 images from each dataset, giving 300 reference-query pairs per hash and
attack setting. Transformation evaluation uses 50 images per dataset.
We report dataset-aggregated rates by pooling the evaluated pairs, and use
random seed 2026 throughout.

\paragraph{Compared variants.}
\emph{Original} uses the original perceptual hash and its native matching
rule. \emph{Smoothing} applies only matching-time randomized smoothing.
\emph{Ours} combines matching-time randomized smoothing with publication-time
hardening. The attacks against Original and Ours optimize against the
corresponding original or hardened reference, respectively.

\paragraph{Attack operating points and norm convention.}
The evaluated attacks use three common operating points, generated with
$\ell_2$ caps of 40, 90, and 180. We measure the
realized $\ell_2$ difference on the full $512\times512\times3$ RGB tensor in
$[0,1]$. Across all hashes, datasets, and Original/Ours attacks, the mean
white-box magnitudes are $39.3$, $89.7$, and $179.2$.

\paragraph{Adaptive white-box attacks.}
The white-box attacker knows the hash, matching rule, threshold, and defense.
We use projected gradient descent in RGB space with a focused flip-margin
objective, 300 steps, step size $1/255$, three random-start restarts, and
early stopping after successful evasion. Projection uses the operating-point
$\ell_2$ cap. Smoothing is evaluated on the resulting adversarial query
using fresh defender randomness.

\paragraph{Adaptive black-box attacks.}
The black-box attacker has no access to target-model gradients and instead
uses Natural Evolution Strategies with the Hamming score returned by target
queries. The perturbation is optimized in a $64\times64$ grayscale space and
mapped to the RGB query. We use a clean start, at most 500 optimization steps,
a 1,500-query limit, 32 antithetic directions per estimate, step size $2/255$,
NES scale $4/255$, and early stopping after successful evasion.

\paragraph{Hardening and smoothing.}
Publication-time hardening uses the shared settings: 200 projected sign-gradient updates, a
$16/255$ hardening constraint, step size $1/255$, and Gaussian
expectation-over-transformation samples at $\sigma=0.10$. At matching time,
we use $\sigma=0.10$, $N_0=100$ independent samples for class selection,
$N=5000$ independent samples for probability estimation, and
$\alpha=10^{-3}$. The reported confidence statement is pointwise for each
certified pair.

\paragraph{Transformation parameters.}
We evaluate JPEG quality levels $\{95,80,60,40\}$; brightness and contrast
factors $\{0.70,0.85,1.15,1.30\}$; retained crop side ratios
$\{0.90,0.80,0.70\}$ followed by resizing; Gaussian-blur kernel sizes
$\{3,5,7\}$; Gaussian-noise standard deviations
$\{0.01,0.02,0.04\}$ in $[0,1]$; and rotations
$\{2^\circ,5^\circ,10^\circ\}$.

\paragraph{Collisions and visual-quality metrics.}
Collision rates use 2,000 randomly sampled non-matching pairs from 100
ImageNet images per hash. Visual-quality evaluation uses 100 images per
dataset and hash. SSIM and LPIPS are computed for paired images and then
averaged; FID compares the corresponding sets of original and
hardened/attacked images and is therefore a distribution-level measurement.

Table~\ref{tab:overall-summary} provides a compact overview of the five
evaluation dimensions. The remaining tables report the complete results used
to construct this summary.

\begin{table*}[tb]
\centering

\small
\setlength{\tabcolsep}{2.0pt}
\renewcommand{\arraystretch}{1.08}
\begin{tabular*}{\textwidth}{@{\extracolsep{\fill}}lcccccccccc}
\toprule
\multirow{2}{*}{Hash} & \multicolumn{2}{c}{White-box} & \multicolumn{2}{c}{Black-box} & \multicolumn{2}{c}{Transform} & \multicolumn{2}{c}{Collision} & \multicolumn{2}{c}{Certified $\overline{R}_2\uparrow$} \\
\cmidrule(lr){2-3}\cmidrule(lr){4-5}\cmidrule(lr){6-7}\cmidrule(lr){8-9}\cmidrule(lr){10-11}
& Original & Ours & Original & Ours & Original & Ours & Original & Ours & Original & Ours \\
\midrule
NeuralHash & 100.0 & \bestcell{20.6} & 35.2 & \bestcell{4.4} & \bestcell{0.0} & 7.7 & \bestcell{0.0} & \bestcell{0.0} & -- & \bestcell{0.2987} \\
DINOHash & 100.0 & \bestcell{14.0} & 15.6 & \bestcell{0.0} & \bestcell{0.5} & 2.8 & \bestcell{0.0} & \bestcell{0.0} & -- & \bestcell{0.2993} \\
HybridHash & 100.0 & \bestcell{1.0} & 22.0 & \bestcell{0.0} & \bestcell{0.0} & 0.5 & \bestcell{0.1} & 0.3 & -- & \bestcell{0.2993} \\
SSCD & 100.0 & \bestcell{3.3} & 73.1 & \bestcell{5.9} & \bestcell{0.5} & 23.6 & \bestcell{0.0} & \bestcell{0.0} & -- & \bestcell{0.2993} \\
SimDINO & 100.0 & \bestcell{43.7} & 16.6 & \bestcell{0.0} & \bestcell{0.1} & 8.9 & \bestcell{0.0} & \bestcell{0.0} & -- & \bestcell{0.2993} \\
ViT2Hash & 100.0 & \bestcell{12.0} & 2.3 & \bestcell{0.0} & \bestcell{0.1} & 3.9 & 0.1 & \bestcell{0.0} & -- & \bestcell{0.2993} \\
C-PDNA & 94.2 & \bestcell{0.0} & \bestcell{0.0} & \bestcell{0.0} & 0.8 & \bestcell{0.3} & 5.9 & \bestcell{5.7} & 0.2968 & \bestcell{0.2993} \\
C-PDQ & 94.7 & \bestcell{0.0} & \bestcell{0.0} & \bestcell{0.0} & \bestcell{0.0} & \bestcell{0.0} & \bestcell{100.0} & \bestcell{100.0} & 0.1282 & \bestcell{0.2993} \\
\bottomrule
\end{tabular*}
\caption{Overall robustness summary. White-box, black-box, transformation, and collision entries are rates (\%); the first two are weighted averages over the three attack operating points reported as $\ell_2$ levels 40, 90, and 180, and transformation rates use the lowest tested intensity. These four error rates are lower-is-better. The certified bound is the mean query-only certified $\ell_2$-radius lower bound $\overline{R}_2$ in normalized-input coordinates and is higher-is-better. A dash means that Original has no model-specific verifier. The best value within each metric and row is highlighted.}
\label{tab:overall-summary}
\end{table*}

\subsection{Five-Dimension Robustness Profiles}

Figure~\ref{fig:supp-radar-profiles} presents the overall results as one radar
plot per perceptual hash. Every axis is oriented so that farther from the
center is better. The certified score is the mean certified radius normalized
by $0.3$; the white-box, black-box, transformation, and collision scores are
one minus their corresponding error rates. The plots make the main trade-off
intuitive: \name greatly expands adversarial and certified robustness for
every hash, while transformation robustness remains model-dependent.
C-PDQ's collapsed collision axis reflects the underlying model's pre-existing
collision behavior.

\begin{figure*}[tb]
  \centering
  \includegraphics[width=0.98\textwidth]{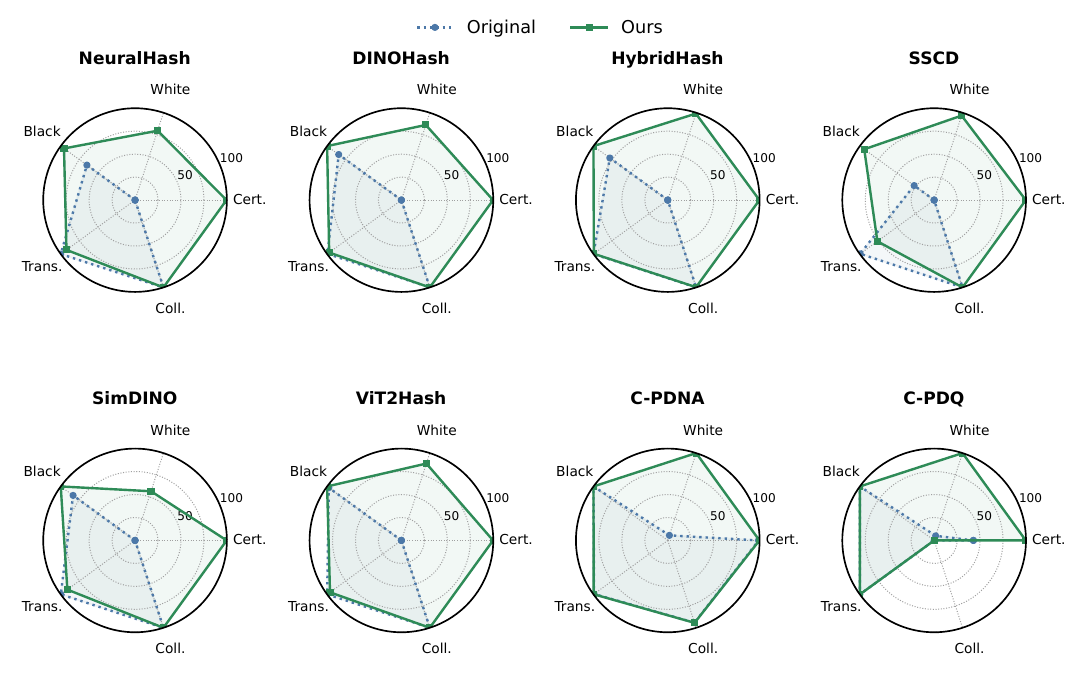}
  \caption{Five-dimensional robustness profiles for all evaluated perceptual
  hashes. Larger values are better on every axis. Original is the unmodified
  hash; Ours combines publication-time hardening and matching-time randomized
  smoothing. White-box and black-box scores are averaged over the three
  attack operating points.}
  \label{fig:supp-radar-profiles}
\end{figure*}

\subsection{Certified Robustness}

For each positive reference-query pair, we estimate the smoothed match
probability using independent Monte Carlo samples and compute the one-sided
Clopper-Pearson lower confidence bound $\underline{p}_A$. The corresponding
query-only certified $\ell_2$ radius is
\begin{equation}
R_2
=
\sigma\Phi^{-1}(\underline{p}_A).
\end{equation}

In addition to the mean certified radius $\overline{R}_2$, we report the
\emph{certified non-evasion rate} (CNER) at selected radii. CNER measures the
fraction of positive pairs that are both classified as matches and formally
certified against every query perturbation within the specified $\ell_2$
radius.

Table~\ref{tab:certified-robustness} reports the complete statistics over 300
positive pairs per hash. Matching-time randomized smoothing provides
certification for all eight hashes, including models without a native
verification procedure. Publication-time hardening further improves the
certificate: the average $\overline{R}_2$ across the eight hashes increases
from approximately $0.252$ with Smoothing to $0.299$ with Ours. At
$R_2=0.2$, Ours achieves a 100\% CNER for all eight hashes.

For the two CertPHash variants, we additionally compare against their native
verification results. C-PDNA provides a mean certified bound of $0.2968$,
which Ours maintains at approximately $0.2993$. For C-PDQ, the mean certified
bound increases from $0.1282$ to $0.2993$. Thus, \name can introduce
certification to previously uncertified hashes and complement hashes that
already provide native certification.

\begin{table*}[tb]
\centering

\small
\setlength{\tabcolsep}{2.8pt}
\renewcommand{\arraystretch}{1.08}
\begin{tabular*}{\textwidth}{@{\extracolsep{\fill}}lccccccccc}
\toprule
\multirow{2}{*}{Hash} & \multicolumn{3}{c}{Certified $\ell_2$ Bound $\overline{R}_2$} & \multicolumn{3}{c}{CNER at $R_2=0.1$} & \multicolumn{3}{c}{CNER at $R_2=0.2$} \\
\cmidrule(lr){2-4}
\cmidrule(lr){5-7}
\cmidrule(lr){8-10}
& Original & Smoothing & Ours & Original & Smoothing & Ours & Original & Smoothing & Ours \\
\midrule
NeuralHash & -- & 0.1685 & \bestcell{0.2987} & -- & 63.7 & \bestcell{100.0} & -- & 51.3 & \bestcell{100.0} \\
\addlinespace[2pt]
DINOHash & -- & 0.2688 & \bestcell{0.2993} & -- & 93.3 & \bestcell{100.0} & -- & 88.7 & \bestcell{100.0} \\
\addlinespace[2pt]
HybridHash & -- & 0.2773 & \bestcell{0.2993} & -- & 97.7 & \bestcell{100.0} & -- & 92.7 & \bestcell{100.0} \\
\addlinespace[2pt]
SSCD & -- & 0.2151 & \bestcell{0.2993} & -- & 75.3 & \bestcell{100.0} & -- & 69.0 & \bestcell{100.0} \\
\addlinespace[2pt]
SimDINO & -- & 0.2182 & \bestcell{0.2993} & -- & 78.3 & \bestcell{100.0} & -- & 69.0 & \bestcell{100.0} \\
\addlinespace[2pt]
ViT2Hash & -- & 0.2655 & \bestcell{0.2993} & -- & 92.7 & \bestcell{100.0} & -- & 86.3 & \bestcell{100.0} \\
\addlinespace[2pt]
C-PDNA & 0.2968 & \bestcell{0.2993} & \bestcell{0.2993} & 99.7 & \bestcell{100.0} & \bestcell{100.0} & 98.3 & \bestcell{100.0} & \bestcell{100.0} \\
\addlinespace[2pt]
C-PDQ & 0.1282 & \bestcell{0.2993} & \bestcell{0.2993} & 90.7 & \bestcell{100.0} & \bestcell{100.0} & 0.0 & \bestcell{100.0} & \bestcell{100.0} \\
\bottomrule
\end{tabular*}
\caption{Query-only certified evasion robustness over 300 positive pairs per hash. $\overline{R}_2$ is the mean certified $\ell_2$-radius lower bound in normalized-input coordinates; CNER is the certified non-evasion rate at $\ell_2$ radii 0.1 and 0.2. Original CertPHash uses deterministic closed-ball CROWN verification, while Smoothing and Ours use open-ball randomized-smoothing certificates. A dash means that Original has no model-specific verifier. Higher is better.}
\label{tab:certified-robustness}
\end{table*}

\subsection{White-Box Evasion Attacks}

Table~\ref{tab:white-box-budget} reports ASR for Original, Smoothing, and Ours
at the three attack operating points. Original is highly vulnerable and
reaches 100\% ASR for most hashes and settings. Smoothing substantially
reduces ASR for some hashes, but its effectiveness varies across models.

Publication-time hardening substantially strengthens Smoothing. At reported
budget 40, Ours has an average white-box ASR of 0.09\%; seven hashes have
0\% ASR and SimDINO has 0.7\%. At budget 90, the average remains 4.6\%.
HybridHash, SSCD, C-PDNA, and C-PDQ have 0\% ASR; NeuralHash, DINOHash, and
ViT2Hash remain below 3\%, while SimDINO reaches 32\%. At budget 180, the
average increases to 30.8\%, with NeuralHash and SimDINO becoming the most
vulnerable.

These results show both the effectiveness and the limitation of empirical
robustness: low observed ASR does not constitute a universal guarantee. The
certified results above provide the complementary formal guarantee.

\begin{table*}[tb]
\centering

\small
\setlength{\tabcolsep}{2pt}
\renewcommand{\arraystretch}{1.02}
\begin{tabular*}{\textwidth}{@{\extracolsep{\fill}}lccccccccc}
\toprule
\multirow{2}{*}{Hash} & \multicolumn{3}{c}{Reported $\ell_2$ Level 40} & \multicolumn{3}{c}{Reported $\ell_2$ Level 90} & \multicolumn{3}{c}{Reported $\ell_2$ Level 180} \\
\cmidrule(lr){2-4}
\cmidrule(lr){5-7}
\cmidrule(lr){8-10}
& Original & Smoothing & Ours & Original & Smoothing & Ours & Original & Smoothing & Ours \\
\midrule
NeuralHash & 100.0 & 43.7 & \bestcell{0.0} & 100.0 & 76.0 & \bestcell{1.7} & 100.0 & 94.3 & \bestcell{60.0} \\
\addlinespace[2pt]
DINOHash & 100.0 & 23.3 & \bestcell{0.0} & 100.0 & 54.7 & \bestcell{2.7} & 100.0 & 81.3 & \bestcell{39.3} \\
\addlinespace[2pt]
HybridHash & 100.0 & 3.3 & \bestcell{0.0} & 100.0 & 21.7 & \bestcell{0.0} & 100.0 & 40.7 & \bestcell{3.0} \\
\addlinespace[2pt]
SSCD & 100.0 & 56.0 & \bestcell{0.0} & 100.0 & 85.0 & \bestcell{0.0} & 100.0 & 98.7 & \bestcell{10.0} \\
\addlinespace[2pt]
SimDINO & 100.0 & 64.7 & \bestcell{0.7} & 100.0 & 86.7 & \bestcell{32.0} & 100.0 & 99.3 & \bestcell{98.3} \\
\addlinespace[2pt]
ViT2Hash & 100.0 & 13.0 & \bestcell{0.0} & 100.0 & 37.0 & \bestcell{0.3} & 100.0 & 72.0 & \bestcell{35.7} \\
\addlinespace[2pt]
C-PDNA & 82.7 & \bestcell{0.0} & \bestcell{0.0} & 100.0 & \bestcell{0.0} & \bestcell{0.0} & 100.0 & \bestcell{0.0} & \bestcell{0.0} \\
\addlinespace[2pt]
C-PDQ & 84.7 & \bestcell{0.0} & \bestcell{0.0} & 99.3 & \bestcell{0.0} & \bestcell{0.0} & 100.0 & \bestcell{0.0} & \bestcell{0.0} \\
\bottomrule
\end{tabular*}
\caption{White-box adversarial evasion ASR (\%) at the three operating points reported as $\ell_2$ levels 40, 90, and 180 in the main paper. Lower is better. The best value within each row and operating point is highlighted.}
\label{tab:white-box-budget}
\end{table*}

\subsection{Black-Box Evasion Attacks}

Table~\ref{tab:black-box-budget} reports the detailed results under the same
three operating points. Ours achieves 0\% ASR across all eight hashes at the
first two settings. At the largest setting, six hashes remain at 0\% ASR;
NeuralHash and SSCD reach 13.3\% and 17.7\%, respectively. Across all hashes
and settings, average black-box ASR decreases from 20.6\% for Original to
1.3\% for Ours.

The stronger performance under black-box attacks is consistent with the
reduced information available to the attacker. Nevertheless, these results
remain empirical and are interpreted separately from the certified
$\ell_2$ guarantees.

\begin{table*}[tb]
\centering

\small
\setlength{\tabcolsep}{2pt}
\renewcommand{\arraystretch}{1.02}
\begin{tabular*}{\textwidth}{@{\extracolsep{\fill}}lccccccccc}
\toprule
\multirow{2}{*}{Hash} & \multicolumn{3}{c}{Reported $\ell_2$ Level 40} & \multicolumn{3}{c}{Reported $\ell_2$ Level 90} & \multicolumn{3}{c}{Reported $\ell_2$ Level 180} \\
\cmidrule(lr){2-4}
\cmidrule(lr){5-7}
\cmidrule(lr){8-10}
& Original & Smoothing & Ours & Original & Smoothing & Ours & Original & Smoothing & Ours \\
\midrule
NeuralHash & 16.3 & 13.7 & \bestcell{0.0} & 28.0 & 25.0 & \bestcell{0.0} & 61.3 & 57.0 & \bestcell{13.3} \\
\addlinespace[2pt]
DINOHash & 3.7 & 3.0 & \bestcell{0.0} & 8.7 & 6.7 & \bestcell{0.0} & 34.3 & 31.0 & \bestcell{0.0} \\
\addlinespace[2pt]
HybridHash & 16.7 & 1.0 & \bestcell{0.0} & 18.7 & 6.0 & \bestcell{0.0} & 30.7 & 19.7 & \bestcell{0.0} \\
\addlinespace[2pt]
SSCD & 49.3 & 43.7 & \bestcell{0.0} & 71.7 & 71.0 & \bestcell{0.0} & 98.3 & 98.3 & \bestcell{17.7} \\
\addlinespace[2pt]
SimDINO & 3.0 & 2.7 & \bestcell{0.0} & 9.3 & 9.3 & \bestcell{0.0} & 37.3 & 37.0 & \bestcell{0.0} \\
\addlinespace[2pt]
ViT2Hash & \bestcell{0.0} & \bestcell{0.0} & \bestcell{0.0} & 1.0 & 0.3 & \bestcell{0.0} & 6.0 & 5.7 & \bestcell{0.0} \\
\addlinespace[2pt]
C-PDNA & \bestcell{0.0} & \bestcell{0.0} & \bestcell{0.0} & \bestcell{0.0} & \bestcell{0.0} & \bestcell{0.0} & \bestcell{0.0} & \bestcell{0.0} & \bestcell{0.0} \\
\addlinespace[2pt]
C-PDQ & \bestcell{0.0} & \bestcell{0.0} & \bestcell{0.0} & \bestcell{0.0} & \bestcell{0.0} & \bestcell{0.0} & \bestcell{0.0} & \bestcell{0.0} & \bestcell{0.0} \\
\bottomrule
\end{tabular*}
\caption{Black-box adversarial evasion ASR (\%) at the three operating points reported as $\ell_2$ levels 40, 90, and 180 in the main paper. Lower is better. The best value within each row and operating point is highlighted.}
\label{tab:black-box-budget}
\end{table*}

\subsection{Robustness to Image Transformations}

Adversarial perturbations are only one way in which published images may be
modified. We therefore evaluate the seven transformations and exact
parameters listed in the protocol above.

Table~\ref{tab:benign} reports the transformation evasion rate at the lowest
tested intensity. Original is generally highly robust to these
transformations. Ours retains low transformation evasion for most hashes but
introduces a model-dependent trade-off. DINOHash, HybridHash, ViT2Hash, and
the two CertPHash variants maintain relatively low rates. SSCD exhibits the
largest degradation, particularly under cropping and rotation, while
SimDINO also becomes more sensitive to geometric transformations.

Improving adversarial robustness therefore does not automatically improve
robustness to structured transformations. Our certified guarantee applies
only to additive norm-bounded perturbations; transformation robustness is an
empirical property.

\begin{table*}[tb]
\centering

\footnotesize
\setlength{\tabcolsep}{2pt}
\renewcommand{\arraystretch}{1.00}
\begin{tabular*}{\textwidth}{@{\extracolsep{\fill}}llccc}
\toprule
Hash & Transformation & Original & Smoothing & Ours \\
\midrule
\multirow{7}{*}{NeuralHash} & brightness & \bestcell{0.0} & 15.5 & 4.2 \\
 & contrast & \bestcell{0.0} & 18.2 & 4.2 \\
 & crop & \bestcell{0.0} & 24.5 & 26.7 \\
 & Gaussian blur & \bestcell{0.0} & 8.2 & 5.0 \\
 & Gaussian noise & \bestcell{0.0} & \bestcell{0.0} & \bestcell{0.0} \\
 & JPEG & \bestcell{0.0} & 0.9 & \bestcell{0.0} \\
 & rotation & \bestcell{0.0} & 21.8 & 14.2 \\
\addlinespace[2pt]
\multirow{7}{*}{DINOHash} & brightness & \bestcell{0.0} & 2.0 & \bestcell{0.0} \\
 & contrast & \bestcell{0.0} & 2.0 & \bestcell{0.0} \\
 & crop & \bestcell{1.3} & 9.5 & 6.4 \\
 & Gaussian blur & \bestcell{0.0} & 0.7 & 0.9 \\
 & Gaussian noise & \bestcell{0.0} & \bestcell{0.0} & \bestcell{0.0} \\
 & JPEG & \bestcell{0.0} & \bestcell{0.0} & \bestcell{0.0} \\
 & rotation & \bestcell{2.0} & 12.2 & 11.9 \\
\addlinespace[2pt]
\multirow{7}{*}{HybridHash} & brightness & \bestcell{0.0} & \bestcell{0.0} & \bestcell{0.0} \\
 & contrast & \bestcell{0.0} & \bestcell{0.0} & \bestcell{0.0} \\
 & crop & \bestcell{0.0} & 3.4 & 2.0 \\
 & Gaussian blur & \bestcell{0.0} & 0.7 & \bestcell{0.0} \\
 & Gaussian noise & \bestcell{0.0} & \bestcell{0.0} & \bestcell{0.0} \\
 & JPEG & \bestcell{0.0} & \bestcell{0.0} & \bestcell{0.0} \\
 & rotation & \bestcell{0.0} & 1.3 & 1.4 \\
\addlinespace[2pt]
\multirow{7}{*}{SSCD} & brightness & \bestcell{0.0} & 13.3 & \bestcell{0.0} \\
 & contrast & \bestcell{0.0} & 10.8 & \bestcell{0.0} \\
 & crop & \bestcell{2.0} & 35.0 & 79.3 \\
 & Gaussian blur & \bestcell{0.0} & 0.8 & 17.2 \\
 & Gaussian noise & \bestcell{0.0} & 1.7 & \bestcell{0.0} \\
 & JPEG & \bestcell{0.0} & 0.8 & \bestcell{0.0} \\
 & rotation & \bestcell{1.3} & 31.7 & 69.0 \\
\addlinespace[2pt]
\multirow{7}{*}{SimDINO} & brightness & \bestcell{0.0} & 6.1 & 4.2 \\
 & contrast & \bestcell{0.0} & 6.8 & 6.2 \\
 & crop & \bestcell{0.7} & 8.3 & 22.9 \\
 & Gaussian blur & \bestcell{0.0} & 3.0 & 10.4 \\
 & Gaussian noise & \bestcell{0.0} & \bestcell{0.0} & \bestcell{0.0} \\
 & JPEG & \bestcell{0.0} & \bestcell{0.0} & \bestcell{0.0} \\
 & rotation & \bestcell{0.0} & 6.1 & 18.8 \\
\addlinespace[2pt]
\multirow{7}{*}{ViT2Hash} & brightness & \bestcell{0.0} & 4.1 & 2.6 \\
 & contrast & \bestcell{0.0} & 2.1 & 3.5 \\
 & crop & \bestcell{0.7} & 4.8 & 12.2 \\
 & Gaussian blur & \bestcell{0.0} & 0.7 & 5.2 \\
 & Gaussian noise & \bestcell{0.0} & \bestcell{0.0} & \bestcell{0.0} \\
 & JPEG & \bestcell{0.0} & \bestcell{0.0} & \bestcell{0.0} \\
 & rotation & \bestcell{0.0} & 3.4 & 3.5 \\
\addlinespace[2pt]
\multirow{7}{*}{C-PDNA} & brightness & \bestcell{0.0} & \bestcell{0.0} & \bestcell{0.0} \\
 & contrast & \bestcell{0.0} & \bestcell{0.0} & \bestcell{0.0} \\
 & crop & \bestcell{0.7} & \bestcell{0.7} & 1.3 \\
 & Gaussian blur & \bestcell{0.0} & \bestcell{0.0} & \bestcell{0.0} \\
 & Gaussian noise & \bestcell{0.0} & \bestcell{0.0} & \bestcell{0.0} \\
 & JPEG & \bestcell{0.0} & \bestcell{0.0} & \bestcell{0.0} \\
 & rotation & 4.7 & \bestcell{0.0} & 0.7 \\
\addlinespace[2pt]
\multirow{7}{*}{C-PDQ} & brightness & \bestcell{0.0} & \bestcell{0.0} & \bestcell{0.0} \\
 & contrast & \bestcell{0.0} & \bestcell{0.0} & \bestcell{0.0} \\
 & crop & \bestcell{0.0} & \bestcell{0.0} & \bestcell{0.0} \\
 & Gaussian blur & \bestcell{0.0} & \bestcell{0.0} & \bestcell{0.0} \\
 & Gaussian noise & \bestcell{0.0} & \bestcell{0.0} & \bestcell{0.0} \\
 & JPEG & \bestcell{0.0} & \bestcell{0.0} & \bestcell{0.0} \\
 & rotation & \bestcell{0.0} & \bestcell{0.0} & \bestcell{0.0} \\
\bottomrule
\end{tabular*}
\caption{Transformation evasion rate (\%) at the lowest tested intensity. Lower is better. The best value within each row is highlighted.}
\label{tab:benign}
\end{table*}

\subsection{Collision Behavior}

A defense should not improve evasion robustness simply by making the matcher
accept more image pairs. We therefore evaluate collision rates on unrelated
images.

As shown in Table~\ref{tab:collision-rate}, collision rates remain unchanged
or near zero for most hashes. NeuralHash, DINOHash, SSCD, and SimDINO remain
at 0\%, while ViT2Hash decreases from 0.15\% to 0\%. HybridHash increases
slightly from 0.05\% to 0.30\%, and C-PDNA remains similar at approximately
6\%.

C-PDQ has a 100\% collision rate for both Original and Ours under the
evaluated configuration. This is an inherited limitation of the underlying
hash configuration rather than a failure introduced by \name. Overall, the
robustness gains are not generally obtained by indiscriminately increasing
the match rate.

\begin{table*}[tb]
\centering

\small
\setlength{\tabcolsep}{4pt}
\renewcommand{\arraystretch}{1.08}
\begin{tabular*}{\textwidth}{@{\extracolsep{\fill}}lccc}
\toprule
Hash & Original & Smoothing & Ours \\
\midrule
NeuralHash & \bestcell{0.00} & \bestcell{0.00} & \bestcell{0.00} \\
\addlinespace[2pt]
DINOHash & \bestcell{0.00} & \bestcell{0.00} & \bestcell{0.00} \\
\addlinespace[2pt]
HybridHash & \bestcell{0.05} & \bestcell{0.05} & 0.30 \\
\addlinespace[2pt]
SSCD & \bestcell{0.00} & \bestcell{0.00} & \bestcell{0.00} \\
\addlinespace[2pt]
SimDINO & \bestcell{0.00} & \bestcell{0.00} & \bestcell{0.00} \\
\addlinespace[2pt]
ViT2Hash & 0.15 & 0.10 & \bestcell{0.00} \\
\addlinespace[2pt]
C-PDNA & 5.90 & 6.05 & \bestcell{5.65} \\
\addlinespace[2pt]
C-PDQ & \bestcell{100.00} & \bestcell{100.00} & \bestcell{100.00} \\
\bottomrule
\end{tabular*}
\caption{Hash collision rate (\%) on sampled non-matching image pairs. Lower is better. The best value within each row is highlighted.}
\label{tab:collision-rate}
\end{table*}

\subsection{Visual Quality after Publication-Time Hardening}
\label{sec:supp-hardening-quality}

Publication-time hardening modifies the image before release. We quantify the
difference between each original and hardened image using SSIM, LPIPS, and
FID. Table~\ref{tab:image-hardening-difference} reports the complete results
across hashes and datasets.

The amount of distortion varies substantially across hash models. C-PDNA and
C-PDQ require relatively small changes, while several learned hashes require
larger modifications at the evaluated hardening strength. We report these
measurements explicitly rather than assuming that the perceptual impact is
identical across models.

Figure~\ref{fig:hardening-examples} complements the full quantitative table
with a favorable matched example for each of the eight hashes. Each column
uses a different high-SSIM ImageNet pair so that the gallery covers diverse
image content as well as every hash model. These examples illustrate the
visual quality that hardening can attain and are not intended to represent
the complete distribution. The remaining visible differences, particularly
for SSCD, are consistent with the model-dependent averages in
Table~\ref{tab:image-hardening-difference}. If visual preservation is a
deployment requirement, the hardening strength can be reduced at the cost of
a smaller robustness gain.

\begin{table*}[tb]
\centering

\small
\setlength{\tabcolsep}{4pt}
\renewcommand{\arraystretch}{1.08}
\begin{tabular*}{\textwidth}{@{\extracolsep{\fill}}llrrr}
\toprule
Hash & Dataset & SSIM & LPIPS & FID \\
\midrule
\multirow{3}{*}{NeuralHash} & ImageNet & 0.742 & 0.262 & 26.8 \\
 & MS-COCO & 0.713 & 0.270 & 75.7 \\
 & Stable-Diffusion & 0.756 & 0.176 & 47.6 \\
\addlinespace[2pt]
\multirow{3}{*}{DINOHash} & ImageNet & 0.699 & 0.301 & 29.6 \\
 & MS-COCO & 0.678 & 0.307 & 85.7 \\
 & Stable-Diffusion & 0.724 & 0.206 & 57.7 \\
\addlinespace[2pt]
\multirow{3}{*}{HybridHash} & ImageNet & 0.734 & 0.258 & 24.9 \\
 & MS-COCO & 0.703 & 0.267 & 70.3 \\
 & Stable-Diffusion & 0.749 & 0.175 & 46.1 \\
\addlinespace[2pt]
\multirow{3}{*}{SSCD} & ImageNet & 0.685 & 0.321 & 56.2 \\
 & MS-COCO & 0.651 & 0.340 & 156.9 \\
 & Stable-Diffusion & 0.698 & 0.234 & 82.5 \\
\addlinespace[2pt]
\multirow{3}{*}{SimDINO} & ImageNet & 0.639 & 0.274 & 23.2 \\
 & MS-COCO & 0.608 & 0.285 & 68.7 \\
 & Stable-Diffusion & 0.660 & 0.182 & 43.7 \\
\addlinespace[2pt]
\multirow{3}{*}{ViT2Hash} & ImageNet & 0.628 & 0.280 & 11.8 \\
 & MS-COCO & 0.605 & 0.295 & 38.5 \\
 & Stable-Diffusion & 0.655 & 0.213 & 26.1 \\
\addlinespace[2pt]
\multirow{3}{*}{C-PDNA} & ImageNet & 0.923 & 0.090 & 5.7 \\
 & MS-COCO & 0.911 & 0.095 & 21.2 \\
 & Stable-Diffusion & 0.933 & 0.053 & 12.7 \\
\addlinespace[2pt]
\multirow{3}{*}{C-PDQ} & ImageNet & 0.989 & 0.037 & 1.7 \\
 & MS-COCO & 0.980 & 0.036 & 8.3 \\
 & Stable-Diffusion & 0.991 & 0.014 & 3.9 \\
\bottomrule
\end{tabular*}
\caption{Visual similarity between original and hardened images. SSIM and LPIPS are averaged over paired images; FID compares the corresponding image sets. Higher SSIM and lower LPIPS/FID indicate better visual preservation.}
\label{tab:image-hardening-difference}
\end{table*}

\begin{figure*}[tb]
  \centering
  \setlength{\tabcolsep}{1pt}
  \renewcommand{\arraystretch}{1.0}
  \begin{tabular}{@{}l*{8}{c}@{}}
    & \scriptsize\textbf{NeuralHash}
    & \scriptsize\textbf{DINOHash}
    & \scriptsize\textbf{HybridHash}
    & \scriptsize\textbf{SSCD}
    & \scriptsize\textbf{SimDINO}
    & \scriptsize\textbf{ViT2Hash}
    & \scriptsize\textbf{C-PDNA}
    & \scriptsize\textbf{C-PDQ} \\
    \scriptsize\textbf{Original}
    & \includegraphics[width=0.103\textwidth]{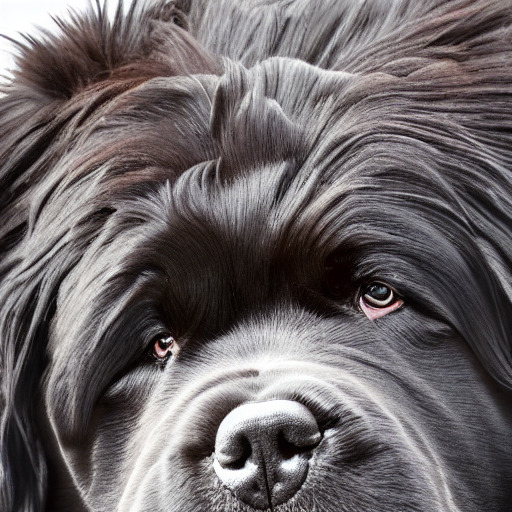}
    & \includegraphics[width=0.103\textwidth]{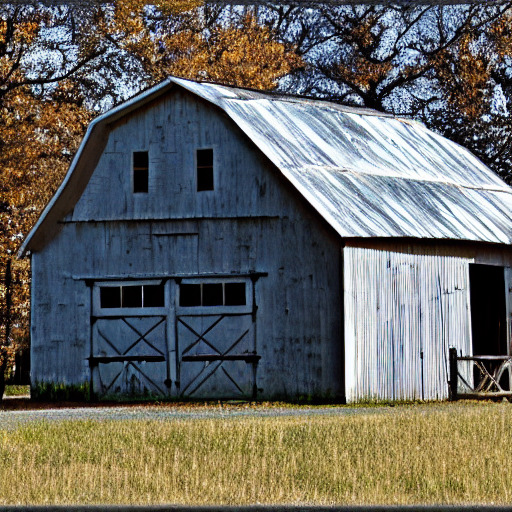}
    & \includegraphics[width=0.103\textwidth]{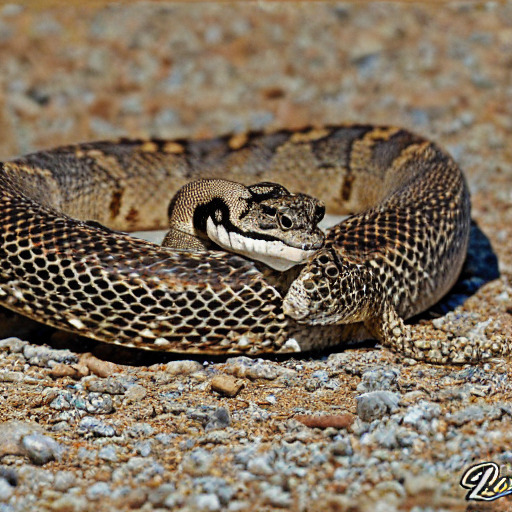}
    & \includegraphics[width=0.103\textwidth]{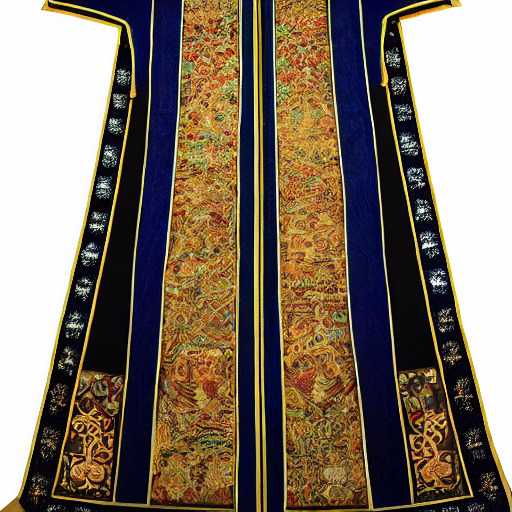}
    & \includegraphics[width=0.103\textwidth]{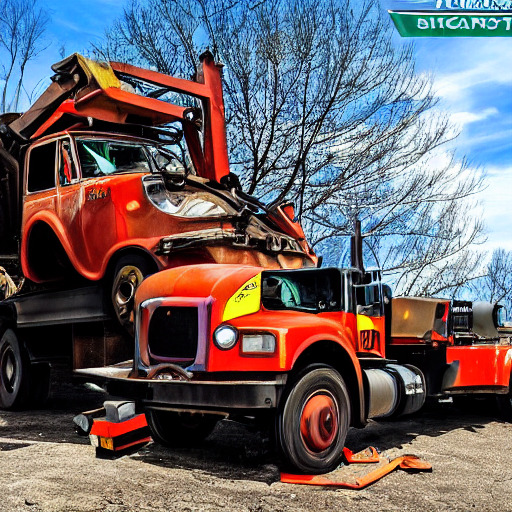}
    & \includegraphics[width=0.103\textwidth]{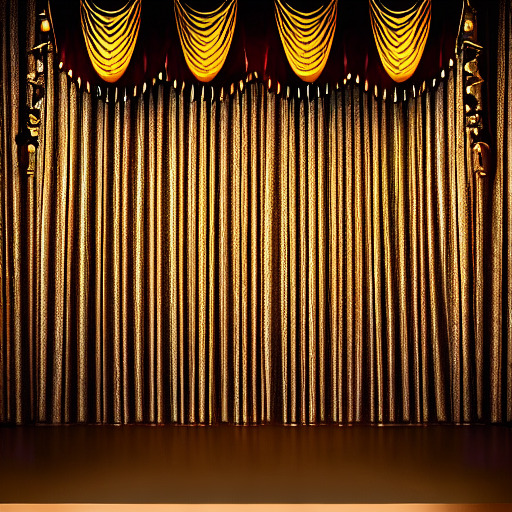}
    & \includegraphics[width=0.103\textwidth]{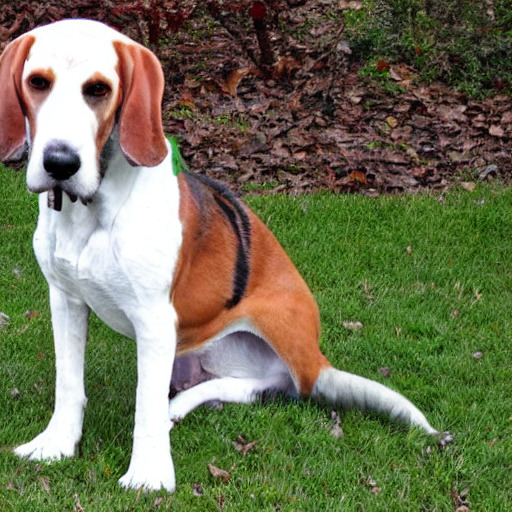}
    & \includegraphics[width=0.103\textwidth]{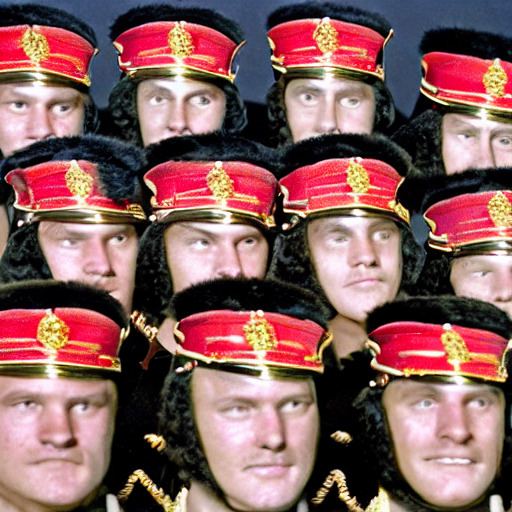} \\
    \scriptsize\textbf{Hardened}
    & \includegraphics[width=0.103\textwidth]{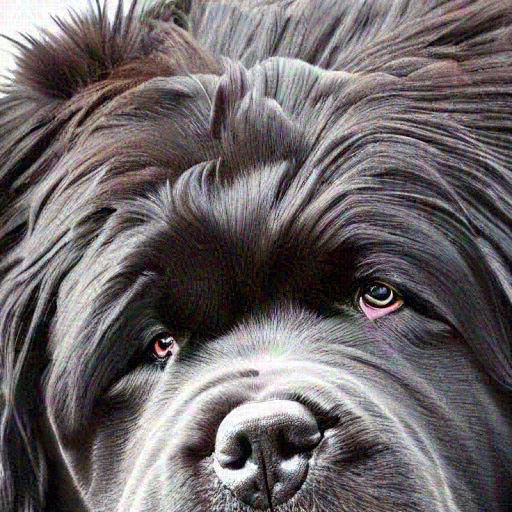}
    & \includegraphics[width=0.103\textwidth]{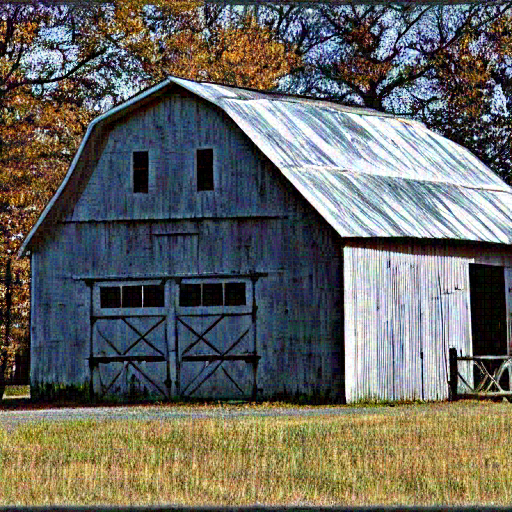}
    & \includegraphics[width=0.103\textwidth]{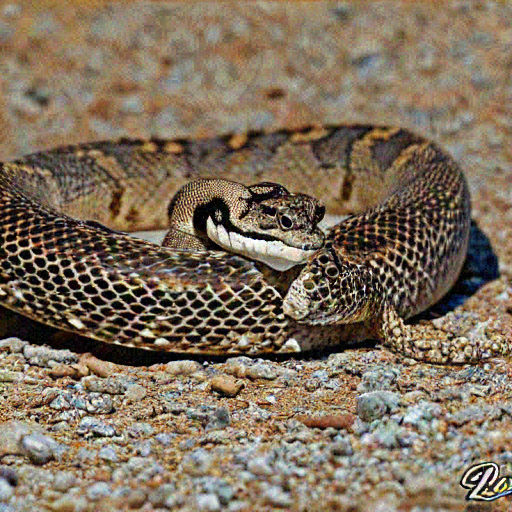}
    & \includegraphics[width=0.103\textwidth]{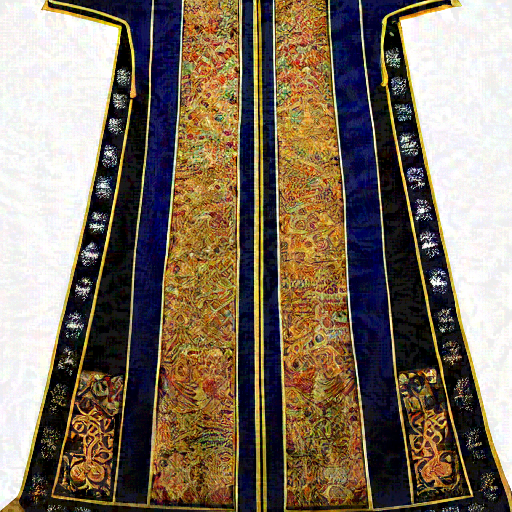}
    & \includegraphics[width=0.103\textwidth]{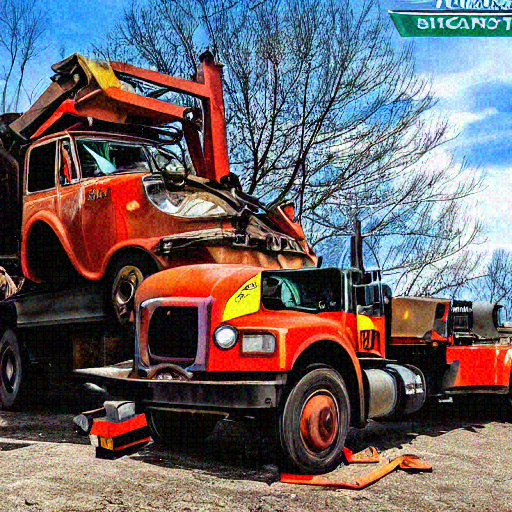}
    & \includegraphics[width=0.103\textwidth]{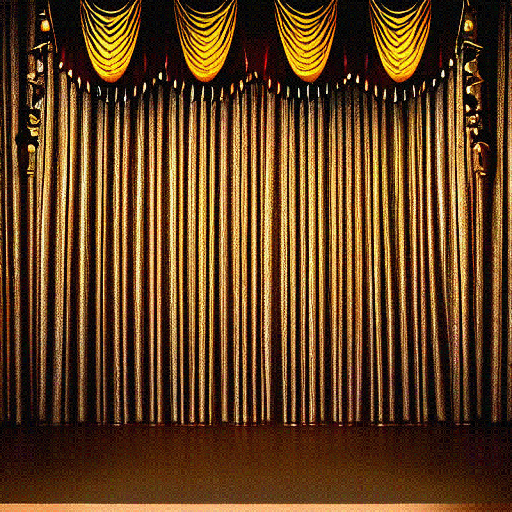}
    & \includegraphics[width=0.103\textwidth]{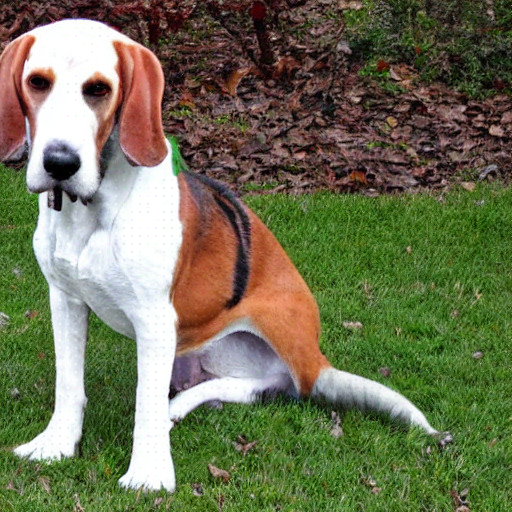}
    & \includegraphics[width=0.103\textwidth]{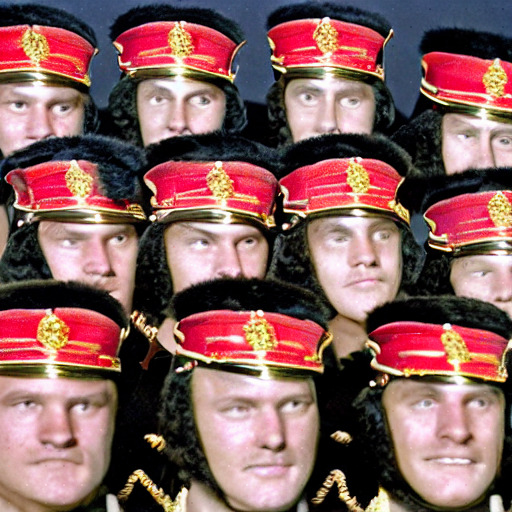}
  \end{tabular}
  \caption{Favorable ImageNet examples before and after publication-time
  hardening for all eight hashes, with eight images in each row. Each column
  is a matched original-hardened pair. The selected examples show attainable visual preservation rather than
  the complete distribution; aggregate results appear in
  Table~\ref{tab:image-hardening-difference}.}
  \label{fig:hardening-examples}
\end{figure*}

\subsection{Image Distortion under White-Box Attacks}

Table~\ref{tab:image-difference} reports SSIM, LPIPS, and FID between each
white-box attacked image and its corresponding original or hardened
reference. As the reported attack level increases from 40 to 180, SSIM
decreases sharply while LPIPS and FID increase for every hash and variant.
This is the visual-quality degradation trend discussed in the main paper.

Across most hashes and operating points, attacks against Ours remain closer
to their corresponding hardened references than attacks against Original
according to these metrics. These measurements are descriptive rather than
robustness guarantees; ASR and certified radius remain the primary
robustness metrics.

\begin{table*}[tb]
\centering

\small
\setlength{\tabcolsep}{2pt}
\renewcommand{\arraystretch}{1.04}
\begin{tabular*}{\textwidth}{@{\extracolsep{\fill}}llccccccccc}
\toprule
\multirow{2}{*}{Hash} & \multirow{2}{*}{Variant} & \multicolumn{3}{c}{Reported $\ell_2$ 40} & \multicolumn{3}{c}{Reported $\ell_2$ 90} & \multicolumn{3}{c}{Reported $\ell_2$ 180} \\
\cmidrule(lr){3-5}
\cmidrule(lr){6-8}
\cmidrule(lr){9-11}
& & SSIM & LPIPS & FID & SSIM & LPIPS & FID & SSIM & LPIPS & FID \\
\midrule
NeuralHash & Original & 0.625 & 0.209 & 22.3 & 0.330 & 0.545 & 56.6 & 0.171 & 0.919 & 103.3 \\
 & Ours & 0.737 & 0.075 & 8.3 & 0.397 & 0.324 & 30.2 & 0.196 & 0.701 & 70.5 \\
\addlinespace[2pt]
DINOHash & Original & 0.625 & 0.209 & 22.9 & 0.330 & 0.545 & 57.5 & 0.171 & 0.919 & 104.1 \\
 & Ours & 0.751 & 0.077 & 10.5 & 0.413 & 0.314 & 29.7 & 0.205 & 0.691 & 71.7 \\
\addlinespace[2pt]
HybridHash & Original & 0.625 & 0.209 & 22.0 & 0.330 & 0.545 & 56.7 & 0.171 & 0.919 & 103.3 \\
 & Ours & 0.742 & 0.069 & 8.1 & 0.401 & 0.309 & 29.0 & 0.198 & 0.686 & 72.2 \\
\addlinespace[2pt]
SSCD & Original & 0.625 & 0.209 & 23.1 & 0.330 & 0.545 & 57.0 & 0.171 & 0.919 & 103.3 \\
 & Ours & 0.755 & 0.084 & 15.0 & 0.417 & 0.374 & 44.6 & 0.206 & 0.783 & 82.7 \\
\addlinespace[2pt]
SimDINO & Original & 0.624 & 0.209 & 22.8 & 0.330 & 0.545 & 57.0 & 0.171 & 0.919 & 103.3 \\
 & Ours & 0.779 & 0.065 & 8.5 & 0.433 & 0.293 & 29.2 & 0.214 & 0.663 & 71.8 \\
\addlinespace[2pt]
ViT2Hash & Original & 0.624 & 0.209 & 21.2 & 0.330 & 0.545 & 56.3 & 0.171 & 0.919 & 103.3 \\
 & Ours & 0.788 & 0.057 & 7.9 & 0.445 & 0.281 & 33.9 & 0.218 & 0.650 & 82.3 \\
\addlinespace[2pt]
C-PDNA & Original & 0.616 & 0.224 & 25.6 & 0.329 & 0.549 & 57.9 & 0.171 & 0.920 & 103.7 \\
 & Ours & 0.636 & 0.142 & 13.5 & 0.337 & 0.460 & 43.1 & 0.173 & 0.846 & 91.5 \\
\addlinespace[2pt]
C-PDQ & Original & 0.624 & 0.211 & 22.0 & 0.330 & 0.546 & 56.8 & 0.171 & 0.919 & 103.4 \\
 & Ours & 0.626 & 0.195 & 18.4 & 0.329 & 0.531 & 52.4 & 0.170 & 0.907 & 98.3 \\
\bottomrule
\end{tabular*}
\caption{Visual similarity between white-box attacked images and their corresponding original or hardened references at the three reported attack levels. SSIM and LPIPS are averaged over paired images; FID compares the corresponding image sets. Higher SSIM and lower LPIPS/FID indicate better visual preservation.}
\label{tab:image-difference}
\end{table*}

\subsection{Ablation: Why Both Shields Matter}

The detailed results consistently show that matching-time randomized
smoothing and publication-time hardening are complementary. Across all
white-box attacks, Smoothing yields an average ASR of approximately 43.8\%,
whereas Ours reduces it to 11.8\%. For black-box attacks, the corresponding
average decreases from approximately 18.0\% with Smoothing to 1.3\% with
Ours. The average certified $\ell_2$ radius also increases from approximately
$0.252$ with Smoothing to $0.299$ with Ours.

Matching-time randomized smoothing stabilizes the match decision and enables
formal certification, while publication-time hardening improves the local
operating point of individual references, making the smoothed matcher
substantially harder to evade.

\end{document}